\pdfoutput=1 
\documentclass[11pt]{article}
\usepackage{amsmath,amssymb,amsthm,mathtools}
\usepackage{dsfont}             
\usepackage[margin=1in]{geometry}
\usepackage[hidelinks]{hyperref}
\usepackage{xurl}               
\usepackage{microtype}
\usepackage{graphicx}
\usepackage[ruled,vlined,linesnumbered]{algorithm2e}
\newcommand{\tr}{\operatorname{tr}}
\newcommand{\diag}{\operatorname{diag}}
\newcommand{\R}{\mathbb{R}}
\newcommand{\E}{\mathbb{E}}
\newcommand{\Oc}{\mathcal{O}}
\DeclareMathOperator*{\argmax}{arg\,max}

\theoremstyle{plain}
\newtheorem{theorem}{Theorem}
\newtheorem{proposition}{Proposition}
\newtheorem{lemma}{Lemma}
\newtheorem{corollary}{Corollary}
\theoremstyle{definition}
\newtheorem{assumption}{Assumption}
\newtheorem{problem}{Problem}
\theoremstyle{remark}
\newtheorem{remark}{Remark}

\title{\bf How Much Does Correctness Cost?\\
Budgeted Placement of Strong Correctors in a Weak Multi-Agent Swarm}
\author{Igor Itkin\\ Independent Researcher, Tel Aviv, Israel\\
\texttt{ig.itkin@gmail.com}\ \ ORCID 0009-0004-9513-8463}
\date{July 2026}

\begin{document}
\maketitle

\begin{abstract}
A cheap swarm of unreliable agents can be steered to a correct consensus by a few strong, expensive
``oracle'' correctors. We ask how much one must spend, and where to place the oracles. We model the swarm as a
consensus on a graph in which each oracle pins one node toward the truth at a cost-coupled, concave strength,
and measure quality by the \emph{coherence} $H(R)=\tr M(R)^{-1}$. Our first result is that $H$ stays
\emph{submodular} (each added oracle helps less than the last) even when the oracles differ in strength, so a
cost-benefit greedy comes within $1-1/e$ of the best placement at any budget. Inverting the budget gives the \emph{budget--correctness frontier}
$B^\star(\varepsilon)$, the least spend that guarantees an $\varepsilon$-correct consensus: closed-form on the
complete graph, and a minimal oracle count $k^\star$ when oracles cost the same. Whether a budget then buys a
few strong oracles or many medium ones is decided by one scalar, the curvature of the cost--quality law:
diminishing returns favour spreading, sharply increasing returns favour concentration. Measured on the Qwen3
ladder ($0.6$--$32$B), the law is concave for factual and math verification (replicated on Gemma-4) but convex
for emergent code tracing, so the verdict is genuinely task-dependent. Code and data:
\url{https://github.com/YehudaItkin/budgeted-oracle-placement}.
\end{abstract}

\section{Introduction}\label{sec:intro}

Deployed multi-agent large language model (LLM) systems increasingly pair two kinds of agent: a large pool
of cheap, individually unreliable workers, and a few strong, expensive models used as verifiers or
correctors. The strong models cost
more but also correct harder. This raises a budgeting question that the consensus literature has not
asked: \emph{given a fixed budget, how many strong correctors should one buy, and where should they be
placed, to guarantee that the swarm reaches a correct answer?}

Practice already treats agent assembly as budgeted selection: Yuan et al.\ cast the choice of tools and
sub-agents as a knapsack over heterogeneous-cost, heterogeneous-quality components and solve it online
\cite{Yuan2025}. That work optimizes an empirically measured success rate with a generic
online-knapsack competitive ratio, over a static set with no interaction structure. We supply the dynamical
foundation that line of work omits. The agents form a consensus loop on a graph, not a static pool, and the objective is the provable tracking error of that loop, not a
measured success rate. The decision is where on the graph each oracle sits, because a node's worth
depends on its position. Finally, the oracle's cost is coupled to its pinning strength through a law $w(c)$, so buying
more correction costs more. This coupling is what turns placement into a budgeting problem. It yields a $(1-1/e)$ submodular guarantee
and a budget--correctness frontier $B^\star(\varepsilon)$, closed-form on the complete graph and
greedy-computable elsewhere. Its equal-cost limit is the minimal oracle count $k^\star$.

The model and the placement machinery extend our companion analysis of delayed verification
\cite{ItkinDelay}, whose homogeneous, cardinality-constrained corrector placement is the special case
$c_i\equiv1$, $w_i\equiv w$ recovered in Corollary~\ref{cor:kstar}.

\noindent This paper makes five contributions.
\begin{itemize}
  \item \textbf{A cost-coupled corrector model} (Section~\ref{sec:model}): a grounded-Laplacian swarm in
        which an oracle of cost $c$ pins toward truth with concave strength $w(c)$; correctness is the
        coherence bound $H(R)\le\varepsilon$, and the object of interest is the minimal budget
        $B^\star(\varepsilon)$.
  \item \textbf{Submodularity under cost-coupled pins} (Theorem~\ref{thm:submod}): the coherence objective
        \cite{BamiehCoherence} is a classically submodular leader-selection criterion \cite{MackinPatterson}; we show the
        error reduction $\rho(R)=H(\varnothing)-H(R)$ stays monotone submodular under heterogeneous,
        cost-coupled diagonal pins (not identical leaders), via the M-matrix /
        entrywise-nonnegative-inverse argument, which is what turns selection into a budgeted knapsack.
  \item \textbf{A budgeted $(1-1/e)$ placement} (Theorem~\ref{thm:knapsack}): cost-benefit greedy buys an
        error reduction within $1-1/e$ of the budget-optimal, concentrating spend on the
        \emph{resolvent centrality per dollar} of each node.
  \item \textbf{The budget--correctness frontier} (Proposition~\ref{prop:KN}, Theorem~\ref{thm:curv},
        Corollary~\ref{cor:kstar}): a closed-form coherence and frontier $B^\star$ on the complete graph;
        its homogeneous-cost limit is a minimal oracle count $k^\star$; and the curvature of the
        cost--quality law as the scalar that decides few-strong versus many-medium oracles.
  \item \textbf{A high-probability (Stage-B) threshold} (Section~\ref{sec:disc}): the same budgeting question
        on a nonlinear majority cascade has a sharp phase transition whose threshold is a
        degree-weighted influence balance, exact on the complete graph, rigorous on dense random graphs
        (Theorem~\ref{thm:stageB}), and hysteretic on sparse ones. The threshold is extended to fallible
        correctors (Proposition~\ref{prop:fallible}), unified with Stage~A as its near-consensus linearization
        (Remark~\ref{rmk:unif}), and demonstrated on a real LLM swarm.
\end{itemize}

\section{Related work}\label{sec:related}
Budgeted corrector placement sits between three literatures; none gives the cost-coupled,
consensus-correctness object we study.

\emph{Leader and pinning placement on consensus.} The spectral backbone is the grounded Laplacian, whose
smallest eigenvalue governs convergence to pinned nodes \cite{PiraniSundaram}. Leader selection minimizes
a consensus error by choosing a fixed number of leaders \cite{LinFardadJov}, and supermodular structure
gives greedy $(1-1/e)$ guarantees \cite{ClarkSubmod}; the coherence (total steady-state variance) objective
\cite{BamiehCoherence} we use is itself a classically submodular leader-selection criterion
\cite{MackinPatterson},
and maximizing the smallest grounded-Laplacian eigenvalue is the closest spectral objective to ours
\cite{WangLambdaMin}. All fix a count of identical agents. Our increment is to couple pin
strength to cost through a quality law $w(c)$, turning selection into a budgeted knapsack, and
to show submodularity survives that cost-coupled heterogeneity.

\emph{Budgeted submodular maximization.} The greedy $(1-1/e)$ bound \cite{NWF}, its single-knapsack
extension \cite{Sviridenko}, and scalable cost-aware variants \cite{Leskovec,KrauseGuestrin} are
off-the-shelf. We supply the missing piece, a consensus value whose cost-coupled heterogeneous form
is submodular, so the knapsack machinery applies and yields the budget--correctness frontier.

\emph{Resilient consensus and trusted nodes.} A few strong nodes can relax the robustness a network needs
to tolerate faults \cite{Abbas}, building on $r$-robustness and the W-MSR filtering rule \cite{WMSR}. These
give a robustness relaxation rather than a sharp budget threshold, a placement guarantee, or a cost
model. Our Stage-B threshold and Stage-A frontier supply all three.

\emph{Budgeted agent selection in LLM systems.} Practice already budgets agent assembly: a composer agent
picks tools and sub-agents by online knapsack \cite{Yuan2025}, model cascades escalate to a strong model
on low confidence \cite{FrugalGPT}, and a recent analysis finds a continuous-gain synergy phase transition
with closed-form compute-allocation rules \cite{Synergy}. That last is the closest: it shares budgeted
allocation and a sharp threshold, but over per-agent compute on correlated inputs, with no graph or
placement. None of these models a consensus on a graph, a corrector placement, a cost-coupled quality
law, or a provable correctness frontier; our threshold is set by graph position (an influence balance
over node degrees), not by count or per-agent compute alone.

\section{Model}\label{sec:model}

\emph{Swarm and correctors.} Let $G=(V,E)$ be a connected graph on $N$ agents, with Laplacian $L=D-A$
($A$ the adjacency matrix, $D$ the diagonal degree matrix). The agents run a linear consensus toward a
truth value $b^\star$: at each step an agent moves toward the average of its neighbours, while a common
anchor of strength $\kappa>0$ holds the loop in place. An unknown subset of agents instead injects a
bounded bias. Stacking those biases into a vector $g\in\R^N$, the swarm settles at an equilibrium that we
want to sit as close to $b^\star$ as possible.

A corrector, or \emph{oracle}, placed at node $i$ pins that node toward $b^\star$ with strength $w_i>0$ (a
larger $w_i$ pulls harder). Let $e_i$ be the $i$th standard basis vector, so $e_ie_i^\top$ is the matrix
that acts only on node $i$. A placement $R\subseteq\Oc$, chosen from a candidate set $\Oc\subseteq V$, adds
one such pin per oracle:
\begin{equation}\label{eq:Mop}
  M(R)=L+\kappa I+\sum_{i\in R}w_i\,e_ie_i^\top = M_0+W_R,\qquad
  W_R=\diag\!\big(w_i\,\mathds{1}[i\in R]\big).
\end{equation}
Here $M_0=L+\kappa I$ is the uncorrected operator and $W_R$ collects the pins; both are positive definite
for $\kappa>0$. The equilibrium error is then $e_\infty=M(R)^{-1}g$. Averaging over where the fault lands
($g$ zero-mean with covariance $\sigma^2 I$) gives an expected squared error of $\sigma^2\tr M(R)^{-2}$. We
work instead with the closely related and standard \emph{coherence}
\begin{equation}\label{eq:H}
  H(R)=\tr M(R)^{-1},
\end{equation}
which also falls each time an oracle is added. We call the loop \emph{$\varepsilon$-correct} when
$H(R)\le\varepsilon$: the smaller $H$, the better the swarm tracks the truth.

\emph{Cost coupled to quality.} The new ingredient is that a stronger oracle is more expensive. Each candidate $i$ carries a cost
$c_i>0$, and its pinning strength is set by a single law of the cost.

\begin{assumption}[diminishing-returns cost--quality law]\label{ass:wc}
There is a strictly increasing, concave $w:\R_{>0}\to\R_{>0}$ with $w_i=w(c_i)$; e.g.\
$w(c)=\bar w\,(1-e^{-c/c_0})$. Concavity encodes that the marginal pinning strength bought per dollar
decreases with spend.
\end{assumption}

Assumption~\ref{ass:wc} is what makes ``buy a few strong or many medium oracles'' a genuine question:
under a budget, concavity of $w$ trades off against the diminishing returns of placement itself.

\emph{The budgeting problems.} With a budget $B>0$ and the error reduction $\rho(R)=H(\varnothing)-H(R)\ge0$, we study the primal and its
inverse:
\begin{problem}[budgeted error reduction]\label{prob:primal}
$\max_{R\subseteq\Oc}\rho(R)$ subject to $\sum_{i\in R}c_i\le B$.
\end{problem}
\begin{problem}[budget--correctness frontier]\label{prob:dual}
$B^\star(\varepsilon)=\min_{R\subseteq\Oc}\sum_{i\in R}c_i$ subject to $H(R)\le\varepsilon$.
\end{problem}
\noindent Problem~\ref{prob:dual} is the headline object: the least one must spend to guarantee an
$\varepsilon$-correct swarm.

\section{Submodular structure}\label{sec:submod}

Adding oracle $i$ to a placement $R$ lowers the coherence by a \emph{marginal gain} $\Delta_i(R)$, which
the Sherman--Morrison identity puts in closed form:
\begin{equation}\label{eq:marg}
  \Delta_i(R):=H(R)-H(R\cup\{i\})
            =\frac{w_i\,\|M(R)^{-1}e_i\|^2}{1+w_i\,e_i^\top M(R)^{-1}e_i}\ \ge\ 0 ,
\end{equation}
This gain is a centrality score read off the resolvent $M(R)^{-1}$: it is largest at high-leverage
nodes, the amplifiers and bridges, where a single pin removes the most error. The reduction $\rho$ is
\emph{submodular} when these marginal gains only shrink as the placement grows: each added oracle removes less
error than it would have earlier, the set-level analogue of diminishing returns. That is what lets a greedy
come within a constant factor of optimal, and our first result is that letting the oracles differ in strength
$w_i$ does not break it.

\begin{theorem}[submodularity under heterogeneous pins]\label{thm:submod}
Let $w_i>0$ for every $i$. Then $\rho(R)=H(\varnothing)-H(R)$ is monotone non-decreasing and submodular, with
$\rho(\varnothing)=0$; equivalently $\Delta_i(R)\ge\Delta_i(S)\ge0$ for all $R\subseteq S\subseteq\Oc$ and
$i\notin S$.
\end{theorem}
\begin{proof}[Proof idea]
Both properties reduce to the sign of the resolvent $M(R)^{-1}$. Monotonicity is the Loewner ordering:
adding a pin gives $M(R\cup\{i\})\succeq M(R)$, hence $H(R\cup\{i\})\le H(R)$. Submodularity holds because
$M(R)$ is a symmetric M-matrix, so its inverse is entrywise nonnegative; this keeps the mixed pin-derivatives
of $H$ nonnegative at every strength, which is why heterogeneous, cost-coupled pins behave like
identical leaders. The computation is carried out in \ref{app:submod}.
\end{proof}

\begin{remark}
The theorem needs only $w_i>0$. This is weaker than Assumption~\ref{ass:wc}, whose concave law
$w_i=w(c_i)$ is a special case, so submodularity holds for every placement problem in this paper. It is also
why the budgeted problem stays tractable: the value $\rho$ is submodular in the set $R$, while the cost
$\sum_{i\in R}c_i$ is modular, so the two interact only through the knapsack constraint. The same statement fails for the worst-case objective
$\lambda_{\min}(M(R))$ and for off-diagonal (edge) interventions, which lack the M-matrix structure; this is
why the contribution is scoped to node placement on the coherence objective. A numerical sweep is
consistent: $0/6000$ submodularity violations on grounded Laplacians, versus ${\sim}170/6000$ once the
M-matrix structure is dropped.
\end{remark}

\section{Budgeted placement}\label{sec:knapsack}

Because $\rho$ is monotone submodular with $\rho(\varnothing)=0$ and the cost $\sum_{i\in R}c_i$ is modular,
Problem~\ref{prob:primal} is a submodular knapsack. Call \emph{cost-benefit greedy} the following procedure of
\cite{Sviridenko,Leskovec}. Enumerate every seed set of at most three oracles; extend each seed by repeatedly
adding the unpinned node of largest ratio $\Delta_i(R)/c_i$ while the budget allows; return the best placement
found. Algorithm~\ref{alg:greedy} states it, and the lazy evaluation of \cite{Leskovec} makes the inner loop
efficient.

\begin{theorem}[$(1-1/e)$ budgeted placement]\label{thm:knapsack}
Let $R^\star_B$ be an optimal solution of Problem~\ref{prob:primal}, and let $R_g$ be the placement returned by
cost-benefit greedy. Then
\begin{equation}\label{eq:knap}
  \rho(R_g)\ \ge\ \big(1-1/e\big)\,\rho(R^\star_B).
\end{equation}
\end{theorem}
\begin{proof}
Immediate from Theorem~\ref{thm:submod} and the single-knapsack guarantee for monotone submodular
maximization \cite{Sviridenko}. See \ref{app:knapsack}.
\end{proof}

The selection rule of Theorem~\ref{thm:knapsack} is interpretable: greedy spends on the node of largest
\emph{resolvent centrality per dollar},
\begin{equation}\label{eq:perdollar}
  \frac{\Delta_i(R)}{c_i}
  =\frac{1}{c_i}\cdot\frac{w(c_i)\,\|M(R)^{-1}e_i\|^2}{1+w(c_i)\,e_i^\top M(R)^{-1}e_i},
\end{equation}
which couples two quantities: the node's graph leverage $\|M(R)^{-1}e_i\|^2$, and the cost-efficiency of
strength $w(c_i)/c_i$. The few-strong-versus-many-medium tradeoff lives in that product.

The same loop also traces the frontier: with the stopping rule $H(R)\le\varepsilon$ in place of the budget,
cost-benefit greedy returns $B^\star(\varepsilon)=\sum_{i\in R}c_i$, the least spend that certifies
$\varepsilon$-correctness.

\begin{algorithm}[t]
\DontPrintSemicolon
\KwIn{grounded operator $M_0=L+\kappa I$; candidates $\Oc$; costs $\{c_i\}$; quality law $w$;
      either budget $B$ (primal) or target $\varepsilon$ (frontier).}
\KwOut{placement $R$; for the frontier, $B^\star(\varepsilon)=\sum_{i\in R}c_i$.}
$R\leftarrow\varnothing$;\quad maintain the resolvent $M(R)^{-1}$ (initialised at $M_0^{-1}$)\;
\While{$\Oc\setminus R\neq\varnothing$}{
  \ForEach{$i\in\Oc\setminus R$}{
     $\Delta_i\leftarrow\dfrac{w(c_i)\,\|M(R)^{-1}e_i\|^2}{1+w(c_i)\,e_i^\top M(R)^{-1}e_i}$
       \tcp*{marginal gain, Eq.~\eqref{eq:marg}}
  }
  $i^\star\leftarrow\argmax_{\,i}\ \Delta_i/c_i$
     \tcp*{resolvent centrality per dollar, Eq.~\eqref{eq:perdollar}}
  \lIf{\normalfont primal \textbf{and} $\sum_{j\in R}c_j+c_{i^\star}>B$}{\Return $R$}
  $R\leftarrow R\cup\{i^\star\}$;\quad rank-1 update of $M(R)^{-1}$ (Sherman--Morrison)\;
  \lIf{\normalfont frontier \textbf{and} $H(R)\le\varepsilon$}{\Return $R$}
}
\Return $R$\;
\caption{Cost-benefit greedy placement / budget--correctness frontier. Prepend the Sviridenko
seed-enumeration of size ${\le}3$ for the $(1-1/e)$ guarantee of Theorem~\ref{thm:knapsack}; use lazy
evaluation \cite{Leskovec} for the inner loop.}
\label{alg:greedy}
\end{algorithm}

\section{The budget--correctness frontier}\label{sec:frontier}

Inverting the greedy bound gives the central object, the least spend that certifies correctness:
$B^\star(\varepsilon)$ from Problem~\ref{prob:dual}. On the complete graph it is exactly solvable in
closed form; the same cost-benefit greedy computes it on any other graph.

\begin{proposition}[exact coherence on the complete graph]\label{prop:KN}
On $K_N$, write $M(R)=\Lambda-\mathbf 1\mathbf 1^\top$ with $\Lambda=\diag(\delta_i)$ and
$\delta_i=N+\kappa+w_i\,\mathds 1[i\in R]$. Then the coherence has the closed form
\begin{equation}\label{eq:HKN}
  H(R)=S_1+\frac{S_2}{1-S_1},\qquad S_1=\sum_i\frac1{\delta_i},\quad S_2=\sum_i\frac1{\delta_i^2}.
\end{equation}
Here $M(R)$ is positive definite for $\kappa>0$, so $S_1<1$ and $H$ is strictly decreasing in each $w_i$.
Inverting \eqref{eq:HKN} yields both the homogeneous minimal count
$k^\star(\varepsilon)=\min\{k:H(k)\le\varepsilon\}$, defined for any $\varepsilon$ above the floor $H(N)$, and
the frontier $B^\star$ over the cost-coupled weights. The formula is exact, matching the direct inverse to
machine precision (proof in~\ref{app:frontier}).
\end{proposition}

Should a fixed budget buy a few strong oracles or many medium ones? On the complete graph the whole answer
is a sharp dichotomy, decided by one scalar of the cost--quality law $w$: its curvature.

\begin{theorem}[curvature dichotomy: spread versus concentrate]\label{thm:curv}
Let a budget $B>0$ be allocated on $K_N$ over oracles of costs $c_i\ge0$ with $\sum_i c_i=B$ and strengths
$w_i=w(c_i)$, and let the placement minimize the coherence \eqref{eq:HKN}. Then the optimum is decided by the
sign of $(N+\kappa+w)\,w''-3(w')^2$ on $[0,B]$.
\begin{itemize}\itemsep2pt
\item[\textup{(i)}] If $w$ is concave, this quantity is negative, and the minimizer is unique and uniform,
$c_i^\star=B/N$. In this case the equal-split coherence $H(m)$ over $m$ oracles of cost $B/m$ is strictly
decreasing in $m$, so the budget optimally buys $N$ medium correctors.
\item[\textup{(ii)}] If instead
\begin{equation}\label{eq:concentrate}
  (N+\kappa+w(c))\,w''(c)\ \ge\ 3\,(w'(c))^2\qquad\text{for all }c\in[0,B],
\end{equation}
then the minimizer places the entire budget on a single oracle.
\end{itemize}
In both cases the $(1-1/e)$ guarantee of Theorem~\ref{thm:knapsack} and the frontier $B^\star(\varepsilon)$ are
unchanged; only the location of the optimum moves.
\end{theorem}
\begin{proof}
Write $g(c)=1/(N+\kappa+w(c))$, so \eqref{eq:HKN} reads $H=S_1+S_2/(1-S_1)$ with $S_1=\sum_i g(c_i)$ and
$S_2=\sum_i g(c_i)^2$. Case~(i) is Theorem~\ref{thm:global}. Under Assumption~\ref{ass:wc} both $g$ and $g^2$
are convex, so a pairwise-balancing exchange strictly lowers $H$ off the equal point and the unique minimizer
is constant; strict monotonicity in $m$ is Theorem~\ref{thm:mmono}. Case~(ii) is Remark~\ref{rem:sharp}, where
the factor~$3$ is shown to be sharp: it is the curvature at which $g^2$ turns concave. Plain convexity of $w$
is not enough, because the leverage term $(w')^2$ keeps the spread optimal below that threshold. The full
argument is in~\ref{app:spread}.
\end{proof}

\begin{corollary}[minimal oracle count]\label{cor:kstar}
Under homogeneous costs $c_i\equiv1$ and strengths $w_i\equiv w$, Problem~\ref{prob:dual} is a pure
cardinality problem and $B^\star(\varepsilon)=k^\star(\varepsilon)$. On the complete graph symmetry makes
greedy exact, so $k^\star(\varepsilon)$ is read off the strictly decreasing $H(k)$ of
Proposition~\ref{prop:KN} directly. This recovers the homogeneous corrector placement of \cite{ItkinDelay}
as the unit-cost limit.
\end{corollary}

\begin{remark}[value guarantee versus budget guarantee]\label{rem:bicriteria}
The $(1-1/e)$ bound of Theorem~\ref{thm:knapsack} is a value guarantee at a fixed budget, not a
budget guarantee. Inverting it gives a value-bicriteria statement: at count $k^\star$ the greedy set
reaches $(1-1/e)$ of the required reduction, while hitting the full target $\varepsilon$ carries the
logarithmic submodular-cover factor.
\end{remark}

\begin{remark}[the deployment reading]
Theorem~\ref{thm:curv} answers the motivating question with one measurable scalar: whether to deploy a few
strong correctors or many medium ones is decided by the curvature of real-model quality-per-dollar. Where
that curvature is concave (the diminishing-returns regime our single-family measurement indicates, with
strength saturating by ${\sim}2$B on the task we measure; Section~\ref{sec:exp}), spreading is optimal under
the model and the ``escalate to one strong model'' heuristic is suboptimal for average correctness;
concentration is justified only where quality is sharply convex in cost (Eq.~\eqref{eq:concentrate}),
which mere convexity does not guarantee.
\end{remark}

\begin{remark}[invariance to the cost axis]\label{rmk:axis}
The verdict is a statement about curvature, so it depends on which resource the budget counts. It is
invariant under any affine rescaling of that axis, and more generally the verdict computed in a
resource $c$ transfers to a resource $s$ with $c=\phi(s)$ exactly when the elasticity product
$\varepsilon_v\,\varepsilon_\phi<1$, where $v(c)$ is the placement value delivered at cost $c$,
$\varepsilon_v=c\,v'(c)/v(c)$ is its elasticity, and $\varepsilon_\phi=s\,\phi'(s)/\phi(s)$. Both inference dollar-cost and energy per token are affine in model size (cost and joules
per token $\propto$ floating-point operations $\propto$ parameters), so $\varepsilon_\phi=1$ and the
parameter-count, dollar, and energy axes give the same verdict. A superlinear dollar cost (large models
disproportionately expensive, $\varepsilon_\phi<1$) only strengthens the case for spreading. The verdict
can invert to concentration only under a strongly concave cost, a per-size volume discount on
large models that standard inference pricing does not exhibit.
\end{remark}

\begin{remark}[biased correctors give an interior optimum]\label{rmk:bias}
The headline ``buy many cheap correctors'' does not degenerate into ``buy infinitely many infinitesimal
ones'': a real corrector pins toward its own verified belief, which carries a bias. Let oracle $i$
pull node $i$ toward $b^\star+\beta_i$ with $\beta_i$ zero-mean, variance $\sigma_\beta^2$, independent of
the fault $g$ (covariance $\sigma^2 I$). The equilibrium error becomes $e_\infty=M(R)^{-1}(g+W_R\beta)$, so
\[
  \E\|e_\infty\|^2=\underbrace{\sigma^2\,\tr M(R)^{-2}}_{\text{tracking (falls with pinning)}}
  \;+\;\underbrace{\sigma_\beta^2\,\tr\!\big(W_R M(R)^{-2}W_R\big)}_{\text{injected bias (rises with pinning)}} .
\]
Per symmetric mode the two terms trade off at a finite optimal strength $w^\star=\sigma^2/(a\,\sigma_\beta^2)$
($a$ the mode's decay rate): perfect correctors ($\sigma_\beta\to0$) recover the pin-as-hard-as-possible
regime of the coherence objective, while any corrector bias caps the useful pinning and makes the optimal
total budget interior. Theorem~\ref{thm:curv} is unaffected: it decides how to allocate
whatever budget is optimal; corrector bias only decides how much to spend.
\end{remark}

Figure~\ref{fig:stageA} shows both halves on a random graph: at a fixed budget the concave (real-LLM)
law is minimized by spreading over many medium correctors while a sharply convex law
(Eq.~\eqref{eq:concentrate}) concentrates on a few
strong ones, and the cost-benefit greedy reaches a lower error per unit budget than degree-based or
random placement.

\begin{figure}[t]
  \centering
  \includegraphics[width=0.92\linewidth]{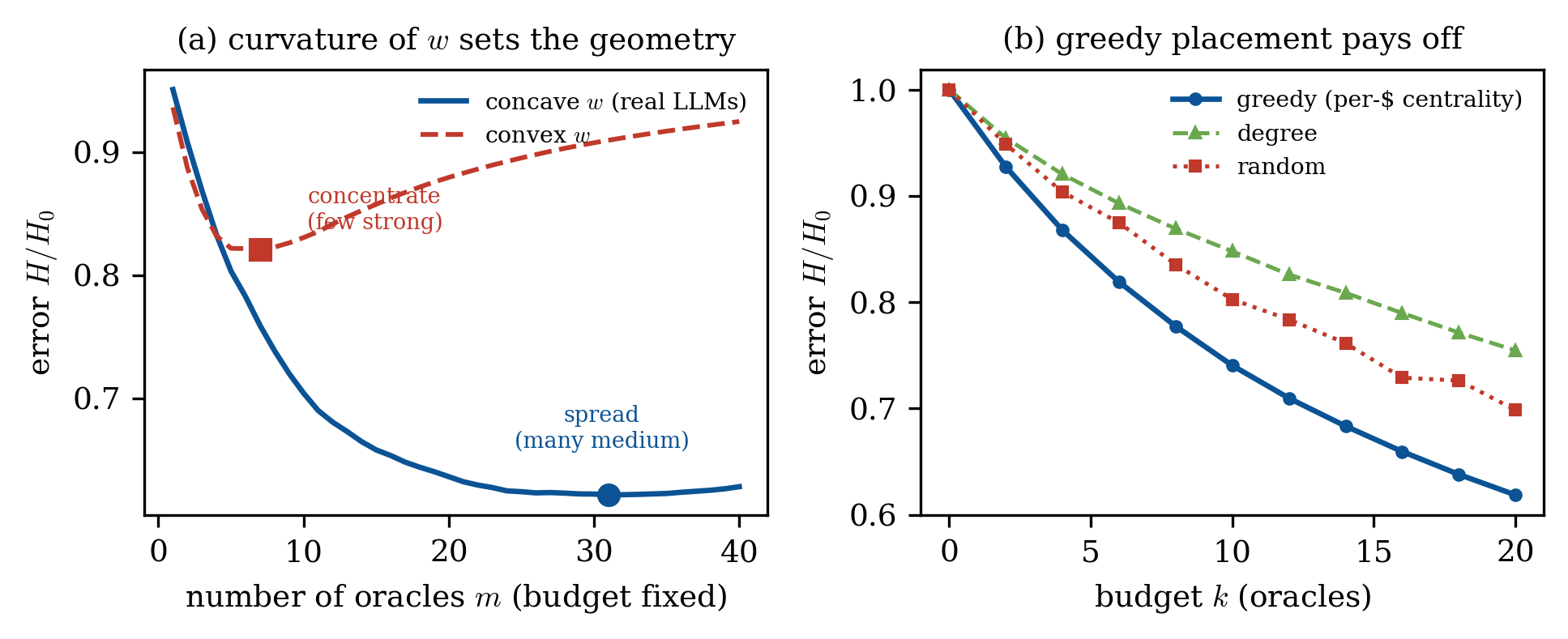}
  \caption{\textbf{Stage~A on a random graph.} (a)~at a fixed budget the truth-tracking error $H/H_0$ is
  minimized by spreading over many medium correctors when $w$ is concave (the regime we measure for Qwen3 in
  Section~\ref{sec:exp}) and by concentrating on a few strong ones when $w$ is sharply convex
  (Eq.~\eqref{eq:concentrate}; Theorem~\ref{thm:curv}). (b)~cost-benefit greedy placement reaches a lower error per unit budget than
  degree-based or random placement (Theorem~\ref{thm:knapsack}).}
  \label{fig:stageA}
\end{figure}

The frontier itself is shown in Fig.~\ref{fig:frontier}: greedy reaches a target error at a smaller
budget than random ($B^\star_{\mathrm{greedy}}<B^\star_{\mathrm{rand}}$), and for the coherence
criterion the minimal budget is a constant fraction of $N$, above the worst-case $\sim\!1/d$.

\begin{figure}[t]
  \centering
  \includegraphics[width=0.92\linewidth]{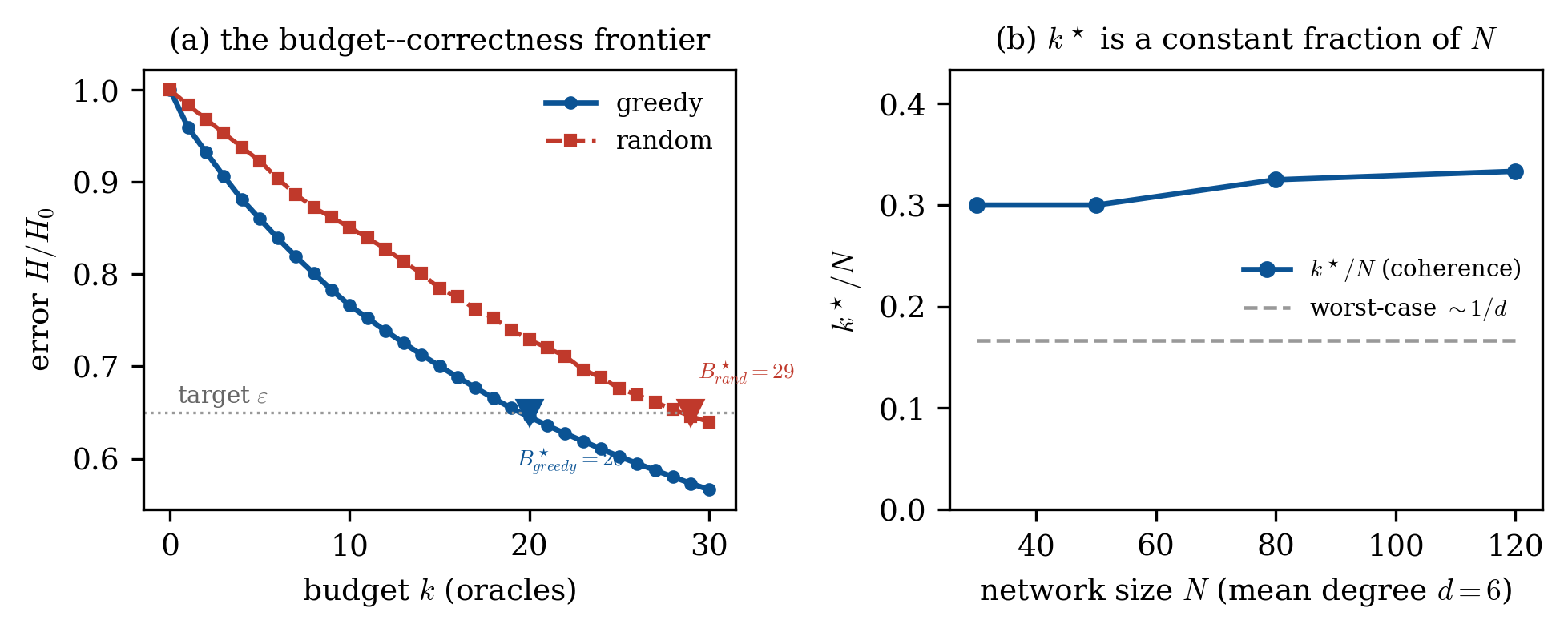}
  \caption{\textbf{The budget--correctness frontier.} (a)~greedy placement reaches a target error
  $\varepsilon$ at a smaller budget than random, so $B^\star_{\mathrm{greedy}}<B^\star_{\mathrm{rand}}$.
  (b)~for the coherence objective the minimal budget $k^\star$ is a constant fraction of $N$ (here
  ${\approx}0.3$ at mean degree $d{=}6$), well above the worst-case $\lambda_{\min}$ scaling $\sim\!1/d$.}
  \label{fig:frontier}
\end{figure}

\section{Measuring the cost--quality law}\label{sec:exp}
Theorem~\ref{thm:curv} makes the deployment answer hinge on one empirical property: the curvature of the
real cost--quality law $w(c)$. Measuring it cleanly means controlling three things that defeat a naive
attempt. We sweep a single model family (five Qwen3 models \cite{Qwen3} from $0.6$ to $14$B, plus a $32$B top-end check) so size is not confounded
with family. We calibrate task difficulty: each verifier is shown the evidence and two candidate
answers, the gold one and a plausible same-topic hallucination drawn from PsiloQA \cite{Psilo}, and must
choose which the evidence supports (a two-alternative forced choice, 2-AFC; $n{=}200$, both answer orders to
cancel position bias); an easier
variant with blatantly contradictory distractors saturated the larger models near $0.97$ and revealed nothing. And
we score by the model's log-probability of the gold choice rather than by parsing its free text,
which otherwise penalizes models that do not follow the answer format. The resulting strength $w$ (the
position-averaged probability of the gold choice) is a proxy for the abstract pin strength $w_i$, a
link we do not otherwise validate.

The measured law is concave. Strength rises steeply from $0.86$ at $0.6$B to $0.94$ at $1.7$B, then
saturates near $0.95$; a $32$B top-end point lands on the same ${\sim}0.96$ plateau, and the $8$--$32$B change-of-slope confidence interval (CI) straddles zero, so the law does not re-accelerate at the top (Fig.~\ref{fig:wc}). A bootstrap test on the change of slope confirms
the diminishing returns: the $0.6$--$4$B change-of-slope confidence interval lies entirely below zero, while
the later segments are flat at the plateau. This concavity is not an artifact of the bounded probability
scale: it survives mapping the strength through a logit transform into an unbounded range, where the
change-of-slope confidence interval stays entirely negative. On this task corrector strength therefore
saturates by about $2$B; past that point extra parameters buy almost no extra correction. This is exactly the concave regime of
Theorem~\ref{thm:curv}, and it makes the deployment reading concrete: a fixed budget buys far more average
correctness as many small-but-adequate correctors than as a few large ones. The verdict is robust to
measurement noise: resampling the per-item scores, refitting the law, and recomputing the budget-optimal
placement on $K_N$, the optimum spreads in $100\%$ of $4000$ bootstrap resamples. The one scope
caveat is the plateau location, not the sign: it is task-dependent, which we now measure directly
rather than leave to conjecture.

To test the location claim instead of asserting it, we grade PsiloQA difficulty by the natural-language-inference (NLI) contradiction
between the gold answer and its hallucinated twin and rerun the $0.6$--$14$B ladder on three disjoint bands:
blatant distractors (easy, contradiction ${\ge}0.5$), subtle ones (hard, $[0.10,0.45)$), and the subtlest
(very hard, $[0.02,0.10)$, sharing no items with the hard band). The plateau forms later as the distractors
get subtler. Between $1.7$ and $4$B the easy and hard curves are flat (paired-bootstrap slope CIs
$[-0.002,+0.008]$ and $[-0.004,+0.011]$), whereas the very-hard curve is still rising there (paired CI
$[+0.005,+0.021]$, excluding zero; Fig.~\ref{fig:diffsweep}). Fitting $w(c)=\bar w(1-e^{-c/c_0})$ to each
band gives a characteristic saturation scale $c_0$ that grows monotonically with difficulty,
$0.20\to0.25\to0.31$B. The curvature stays concave on all three bands. We use a paired bootstrap because the
two model sizes score the same items (per-item correlation ${\approx}0.73$ on this band); under the more
conservative independent bootstrap the very-hard slope is only marginal (CI $[-0.002,+0.028]$), so we read
the difficulty-dependence of the plateau location as suggestive, not decisive. The concave
sign is what the deployment verdict needs, and that is robust across all three bands.

\begin{figure}[t]
  \centering
  \includegraphics[width=0.62\linewidth]{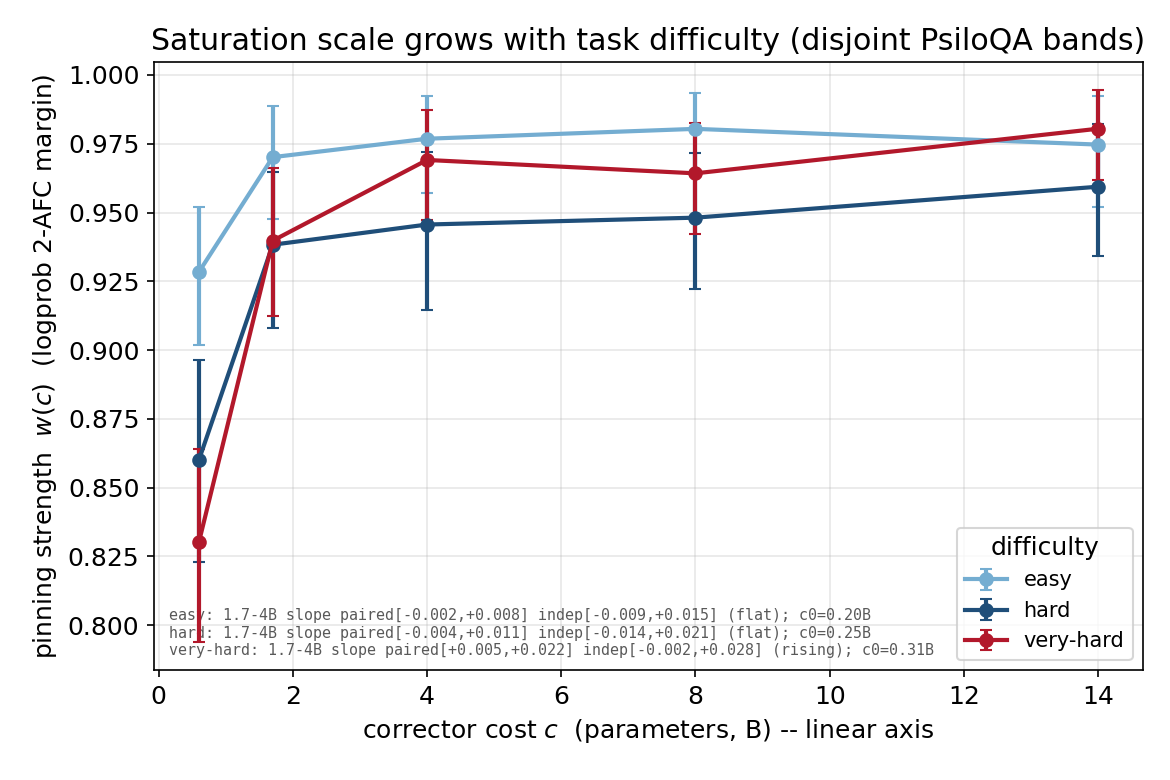}
  \caption{\textbf{The saturation scale grows as the task gets harder.} The same Qwen3 ladder
  ($0.6$--$14$B, logprob 2-AFC, $n{=}200$) on three disjoint PsiloQA difficulty bands, graded by the
  gold--hallucination NLI contradiction. Easy and hard distractors are flat between $1.7$ and $4$B; the
  subtlest band is still rising there (paired-bootstrap slope CI excludes zero; marginal under an independent
  bootstrap). The fitted scale $c_0$ in $w=\bar w(1-e^{-c/c_0})$ grows $0.20\to0.25\to0.31$B. Error bars are
  $95\%$ bootstrap CIs. The curvature is concave in every band, so spreading stays optimal; only where
  the plateau sets in shifts.}
  \label{fig:diffsweep}
\end{figure}

\begin{figure}[t]
  \centering
  \includegraphics[width=0.62\linewidth]{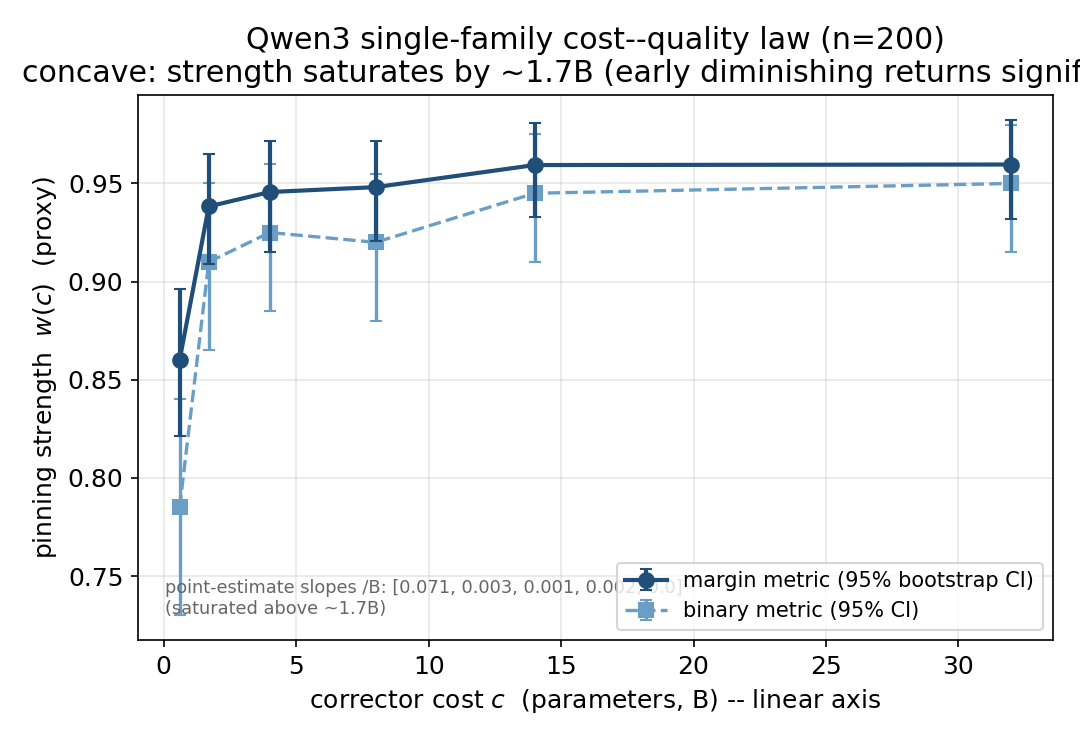}
  \caption{\textbf{The measured cost--quality law is concave (Qwen3, $0.6$--$32$B, $n{=}200$).} Each model
  chooses, from the evidence, between the gold answer and a plausible hallucination (PsiloQA); $w$ is the
  model's probability of the gold choice (scored from token log-probabilities), position-averaged over both answer orders (solid: continuous, $95\%$
  bootstrap CI; dashed: robust pick in both orders). Strength rises steeply to ${\sim}1.7$B then saturates;
  the $0.6$--$4$B change-of-slope CI is below zero (significant diminishing returns) and the later segments
  are flat. The single family removes the model-family confound and the linear axis does not flatter the
  curvature. By Theorem~\ref{thm:curv} this concavity makes spreading the budget optimal.}
  \label{fig:wc}
\end{figure}

\emph{A second family.} The concave saturation is not a Qwen idiosyncrasy. Rerunning the identical
$2$-AFC protocol on the same PsiloQA hard band with the architecturally distinct Gemma-4 ladder
\cite{Gemma4} (E2B/E4B/$12$B, plus a $31$B top-end check) reproduces it: strength plateaus near ${\sim}0.95$,
flat from $2$ to $31$B (Fig.~\ref{fig:crossfamily}). If anything the newer family saturates earlier, already
at the ceiling by $2$B, so extra parameters buy even less correction than on Qwen3 and the deployment
verdict (spread over many medium correctors) holds a fortiori. This meets the single-family caveat with a
second vendor and architecture.

\emph{A different task.} Family and difficulty are not the only axes; the task type moves the plateau
too. On a math-reasoning $2$-AFC built from GSM8K (the worked solution as evidence, a plausible
arithmetic-slip distractor) the Qwen3 ladder is again concave, with the $1.7$--$4$--$8$B change-of-slope CI
entirely negative ($[-0.038,-0.012]$). It saturates later and higher than factual recall: the
significant gains run through $4$B and the plateau sits near ${\sim}0.98$ rather than ${\sim}0.95$. The
concave sign, and with it the spreading verdict, survives the task change; only where the
plateau sets in shifts, consistent with the difficulty sweep. A code-tracing $2$-AFC (predict a program's
printed output), a task needing a capability that emerges only with scale, realizes the
other, convex branch: on demanding programs the whole ladder stays at chance ($0.49$ to $0.57$ over
$0.6$--$32$B, no usable verifier), while on tractable programs the law is sigmoidal: flat at chance
through $1.7$B, then a convex emergent rise ($0.49\to0.56\to0.60\to0.72$ from $1.7$ to $14$B, robust
both-order accuracy climbing $0.04\to0.55$) before saturating at ${\sim}0.72$. So the concave (spread) and
convex (concentrate) regimes of Theorem~\ref{thm:curv} are both empirically instantiated, and the curvature,
and with it the deployment verdict, is genuinely task-dependent: factual and math verification saturate
concavely (spread), while emergent code tracing rises convexly (concentrate up to the knee).

\begin{figure}[t]
  \centering
  \includegraphics[width=0.62\linewidth]{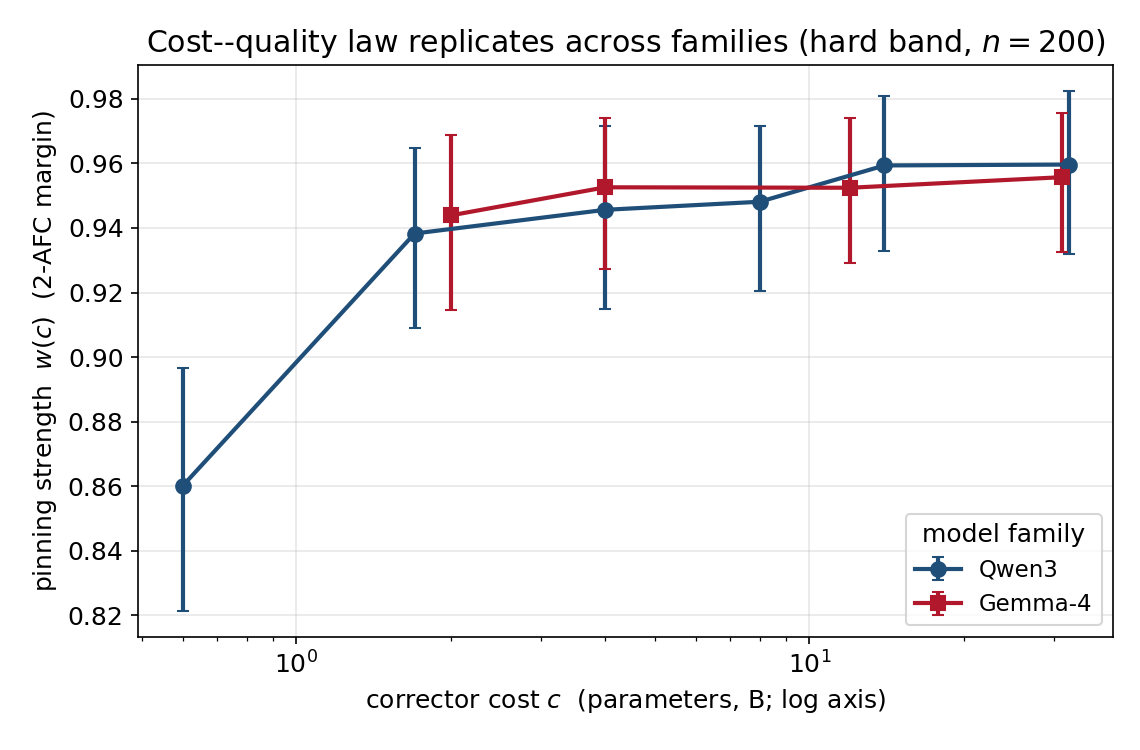}
  \caption{\textbf{The concave law replicates across families.} The Qwen3 ($0.6$--$32$B) and Gemma-4
  (E2B/E4B/$12$/$31$B) ladders on the same PsiloQA hard band, same $2$-AFC logprob protocol ($n{=}200$). Both
  rise then plateau near ${\sim}0.95$; the newer Gemma-4 saturates earlier. Error bars are $95\%$ bootstrap
  CIs; $x$-axis is nominal size (log).}
  \label{fig:crossfamily}
\end{figure}

The full budget--correctness sweep (greedy versus random placement against $B^\star(\varepsilon)$, and
the Stage-B phase transition of Section~\ref{sec:disc}) calibrates only $\kappa$ and the measured
$w(c)$; the theory remains the primary deliverable.

\section{Stage B: the high-probability frontier}\label{sec:disc}
Stage~A studied a linear consensus and its average error. Stage~B asks the same budgeting
question (how many oracles, and where) of a nonlinear majority dynamics, and asks it in a
stronger form: not whether the average error is small, but whether truth wins with high probability.
The two stages model one deployment problem from two angles. They are in fact one dynamical system: Stage~A
is the linearization of Stage~B about the truth-consensus, so the same grounded operator $M(R)$ governs both
(Remark~\ref{rmk:unif}); far from consensus, though, the nonlinear threshold studied here is a genuinely
sharper, non-spectral object.

The model is a majority cascade. Each agent $j$ holds a $\pm1$ belief $x_j$ and, at each step, copies the
majority vote of its neighbours, correctly with reliability $p_{\mathrm r}\in(0,1)$ and at random otherwise. A
set $R$ of $k$ oracles is pinned to $+1$ (truth), a set $F$ of $f$ seeds is pinned to $-1$ (a planted
falsehood), and the remaining agents are free; write $\rho_R=k/N$ and $\rho_F=f/N$ for the oracle and seed
fractions. We study two initial conditions for the free agents: a
\emph{balanced start}, where each free agent is $\pm1$ independently with equal probability (initial
magnetization $O(1/\sqrt N)$, so no starting majority), and an \emph{established-falsehood start}, where every
free agent begins at $-1$. Throughout, ``with high probability'' means with probability $1-o(1)$ as
$N\to\infty$. The question is whether truth wins the final consensus. Simulation shows a sharp \emph{phase transition} in the oracle count
(Fig.~\ref{fig:phase}): the probability that truth wins jumps from near $0$ to near $1$ at a critical count.
That critical fraction $k^\star/N$ barely depends on $N$, and the jump sharpens as $N$ grows, a
genuine thermodynamic threshold rather than a smooth crossover. Concentrating the oracles on high-leverage nodes
roughly halves the number needed.

On the complete graph the threshold is the count balance $k>f$, sharp in the fraction (a margin $k-f\gg
\sqrt N$; Proposition~\ref{prop:stageB}). On a general graph it becomes a degree-weighted \emph{influence
balance}, $D_R>D_F$, where $D_R=\sum_{i\in R}d_i$ and $D_F=\sum_{j\in F}d_j$ are the total degrees of the
oracle and seed sets and $d_i$ is the degree of node $i$: what counts is not the number of oracles but their
total degree. Random placement therefore needs $k^\star$ just above $f$,
while placing them on the highest-degree nodes needs only $k^\star\approx f\,\bar d/d_{\max}$, with $\bar d$
the mean degree and $d_{\max}$ the largest (Fig.~\ref{fig:phase}, within ${\approx}5\%$). This is the placement counterpart of the
continuous-gain synergy threshold of \cite{Synergy}.

\emph{Real agents confirm it.} The leverage advantage is not only a feature of the idealized dynamics. In a
content-free opinion-consensus swarm of real LLM agents (Qwen3-1.7B, $N{=}25$ on a heavy-tailed
Barab\'asi--Albert graph with $\bar d/d_{\max}=0.26$, $40$ trials), correctors placed on the highest-degree
nodes reach truth-consensus at $k^\star{\approx}1.7$, while random placement needs $k^\star{\approx}7$
(Fig.~\ref{fig:placement}), a $4\times$ saving that matches the predicted ratio $\bar d/d_{\max}$.
Isolating the dynamics is what makes the effect visible: on factual questions a grounded agent leans
on its own evidence rather than its neighbours, washing the placement effect out; the content-free vote
exposes the majority-copy mechanism the threshold describes.

\emph{The correctors can be real reasoning agents, not stubborn pins.} The swarm above fixes each corrector to
the truth by fiat. We can instead let every corrector be an actual reasoning agent: a Qwen3-4B model that reads
the question with its grounding evidence, thinks in a chain of thought, and commits to an answer. Its
reliability is then whatever the model achieves, an emergent $\bar q=0.893$ on the hard hallucination band
rather than a knob (\texttt{swarm\_llm\_correctors.py}; $N{=}25$, $f{=}6$, $30$ question draws from a
$200$-item pool, truth balanced across A and B within each corrector's ten-sample batch so position is not a
confound). That reliability is bimodal: $23$ of the $30$ items are answered unanimously and the rest fall
between $0.3$ and $0.9$. Seating these real correctors on the high-degree nodes reaches the $P{=}1/2$
majority-crossing with much less budget than random placement: $k^\star{=}3.8$ under leverage versus
$k^\star{=}10$ under random. The separation is significant even at this size (leverage over random, $z{=}2.98$
at $k{=}4$ and $z{=}2.36$ at $k{=}6$; Wilson $95\%$ intervals in Fig.~\ref{fig:placement}). The leverage curve
reaches the coin-flip level rather than certainty, so we read $k^\star$ as the budget to parity, not to
truth. Against an i.i.d.\ Bernoulli$(\bar q)$ pin, the synthetic fallible corrector of
Fig.~\ref{fig:swarm_mech}a, the real agents match under random placement (both cross only at the boundary
$k{=}10$). Their reasoning errors are, however, correlated by question difficulty rather than
independent, as documented for LLMs more broadly~\cite{CorrErr}; the per-trial correctness variance is $4.7\times$ the i.i.d.\ value (intraclass
correlation $\rho\approx0.41$), so on an easy item every corrector agrees and the hubs carry the truth together.
At $n{=}30$ this correlation does not resolve a significant advantage over the i.i.d.\ proxy under leverage, so
we claim only that real reasoning correctors reproduce the leverage effect, not that they beat the synthetic
model.

\begin{figure}[t]
  \centering
  \includegraphics[width=0.6\linewidth]{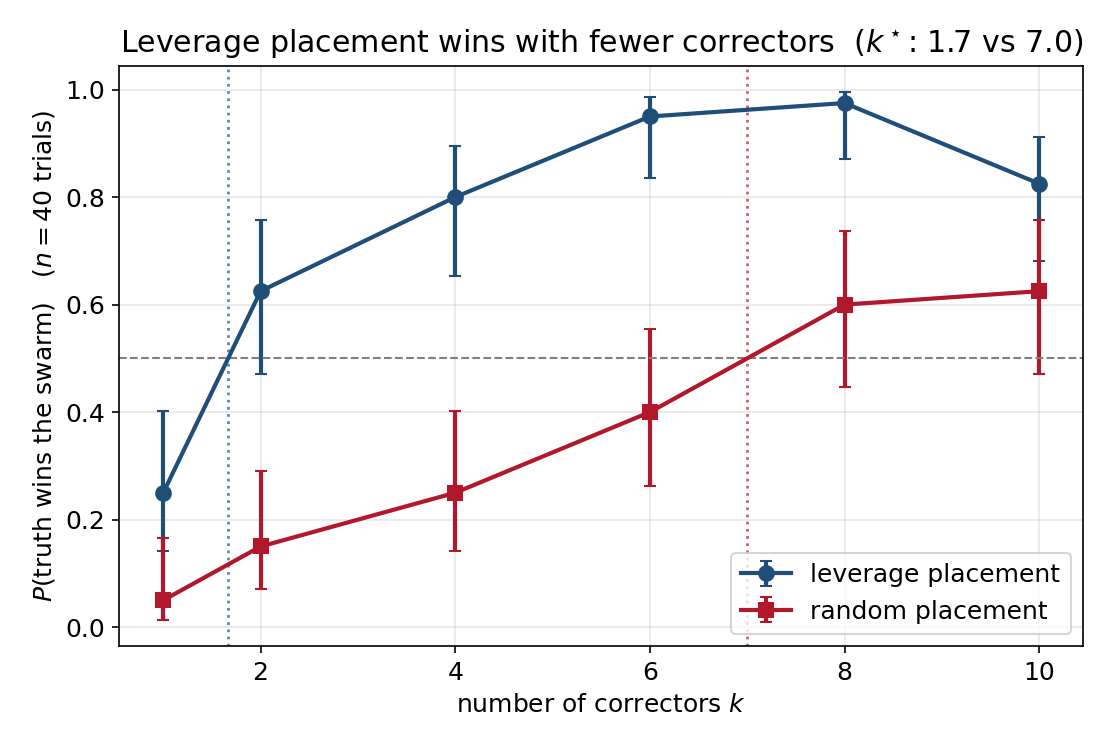}
  \caption{\textbf{Leverage placement in a real LLM swarm.} $P(\text{truth wins})$ versus corrector count
  $k$ for high-degree (leverage) and random placement in a content-free opinion-consensus swarm
  (Qwen3-1.7B, $N{=}25$, $f{=}6$ false seeds, $40$ trials; Wilson $95\%$ intervals). Leverage crosses
  $P{=}1/2$ at $k^\star{\approx}1.7$ versus ${\approx}7$ for random, matching the predicted
  $\bar d/d_{\max}$.}
  \label{fig:placement}
\end{figure}

\emph{Spread, or concentrate?} The two stages advise along orthogonal axes, and ``concentrate''
means different things in each. Stage~A sets the granularity of spend: under a concave cost--quality
law, split the budget into many medium oracles rather than a few strong ones. Stage~B sets the
position: seat whatever oracles one buys on the highest-leverage (highest-degree) nodes. Stage~A's
``concentrate'' (few strong oracles, under convex $w$) is a claim about strength per oracle; Stage~B's
``concentrate on high-degree'' is a claim about which nodes to occupy, so the two recommendations are not in
direct conflict. They compose under a precise condition, which we can now state because the fallible
generalization below (Proposition~\ref{prop:fallible}) supplies the missing ingredient: Stage-B oracles
with a cost-coupled reliability rather than infallible hard pins. By Corollary~\ref{cor:compose}, a Stage-A
placement wins the Stage-B cascade with fallible correctors of reliability $q_i$ (their measured accuracy) exactly when the
reliability-weighted balance holds and no single fallible corrector's degree exceeds the residual
truth margin it leaves. Spreading the budget over many medium correctors keeps that safety margin satisfied,
so ``many medium correctors on high-degree nodes'' is composition-safe; concentrating a medium-reliability
corrector on one dominant hub is not: it broadcasts its own error rate to the whole swarm.

\begin{proposition}[fallible, strength-weighted correctors]\label{prop:fallible}
Let each corrector $i\in R$ hold the truth with reliability $q_i\in(\tfrac12,1]$ and each seed $j\in F$ hold
the falsehood with reliability $q_{\mathrm{seed}}$, so a pinned node emits mean spin $2q-1$. In the
dense-graph regime of Theorem~\ref{thm:stageB} ($pN\ge C\log N$), the free population reaches the
truth-consensus if and only if
the reliability-weighted influence balance
\begin{equation}\label{eq:wbalance}
  \sum_{i\in R} d_i\,(2q_i-1)\ >\ \sum_{j\in F} d_j\,(2q_{\mathrm{seed}}-1)
\end{equation}
holds and no single corrector dominates the free agents it feeds: $d_i^{\mathrm{free}}\le \mu_{-i}$, where
$d_i^{\mathrm{free}}$ is the number of free agents corrector $i$ feeds and $\mu_{-i}$ is the net signed margin
those agents receive from all other pins.
\end{proposition}

We identify $q_i$ with the measured $2$-AFC accuracy of Section~\ref{sec:exp}. The balance has three readings.
On $K_N$ it reduces to a sum-of-margins test, $\sum_{i\in R}(2q_i-1)>\sum_{j\in F}(2q_{\mathrm{seed}}-1)$, so
extra correctors help only when their reliability margins beat the seeds'. As $q_i\to1$ it factorizes to the
infallible balance $D_R>D_F$ (\ref{app:stageB}). The dominance condition is what makes composition safe: if it
fails, a corrector that draws wrong flips the whole free population, so the consensus tracks that corrector's
reliability rather than the truth.

\begin{corollary}[when the two stages compose]\label{cor:compose}
A Stage-A budgeted placement $R$ also wins the
Stage-B cascade iff \eqref{eq:wbalance} holds and $d_i^{\mathrm{free}}\le \mu_{-i}$ for every fallible
corrector. The concave-$w$ ``spread'' prescription of Theorem~\ref{thm:curv} keeps the second condition
satisfied; a high-degree node given only medium reliability violates it and can invert the balance.
\end{corollary}

\emph{Real agents confirm the mechanism.} Two checks in a real weak-LLM swarm bear out
Proposition~\ref{prop:fallible} (Fig.~\ref{fig:swarm_mech}). (a)~When the correctors are made fallible at a
tunable reliability $q$ (free agents Qwen3-1.7B copying neighbours), the truth-win probability rises
linearly in the effective strength $2q-1$ ($r=0.98$, crossing $\tfrac12$ near $q\approx0.7$): the
measured reliability enters the swarm exactly as the $2q-1$ weight of \eqref{eq:wbalance}, closing the
accuracy-to-pin-strength proxy gap empirically. Feeding in the actual measured hard-band strengths
rather than a synthetic knob reproduces this end to end: truth-win probability climbs with the ladder value
($0.70$ for a $0.6$B corrector at $w{=}0.86$ up to ${\sim}0.9$ at $14$B, $w{=}0.96$; correlation $0.94$), so
the measured $w(c)$ predicts the real swarm outcome. (b)~The leverage-placement advantage is present only when
agents actually copy their neighbours: as grounding rises from a content-free vote to full evidence, the
gap between high-degree and random placement collapses from $+0.75$ to $0$, since grounded agents self-correct
and graph position stops mattering.

\begin{figure}[t]
  \centering
  \includegraphics[width=0.98\linewidth]{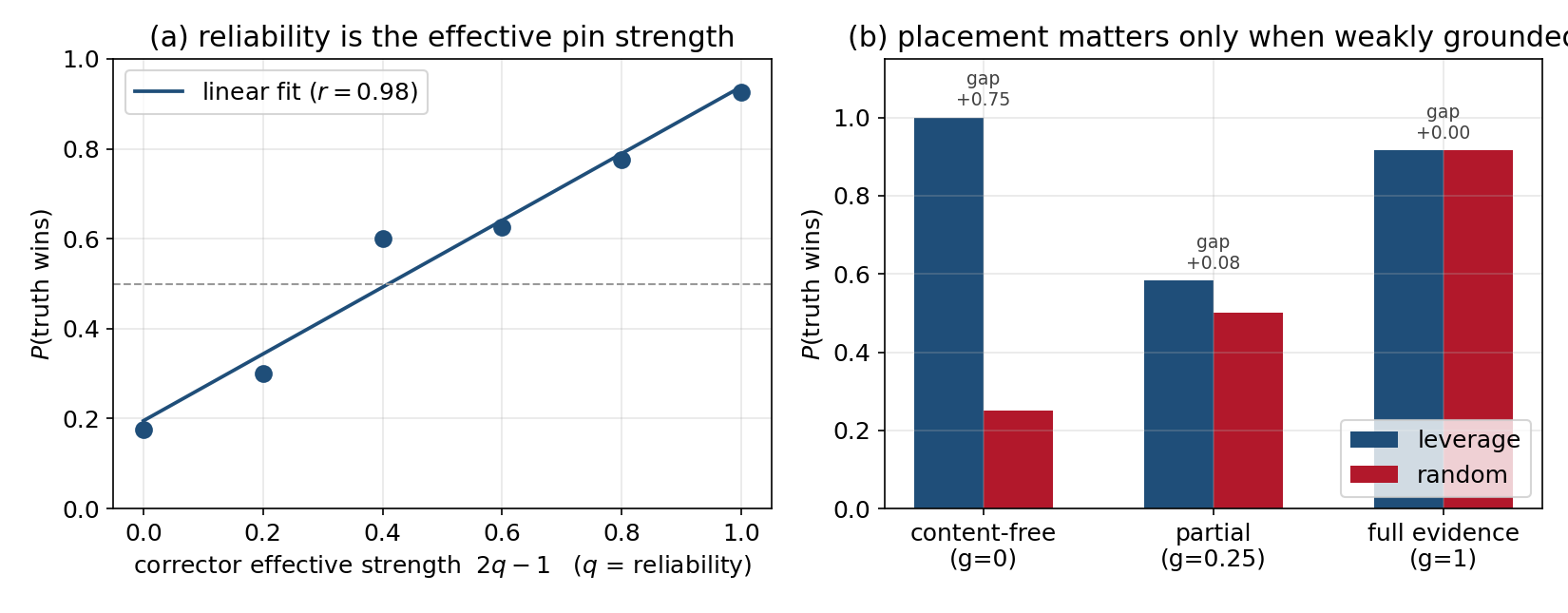}
  \caption{\textbf{Two mechanism checks in a real LLM swarm.} (a)~With fallible correctors, $P(\text{truth
  wins})$ is linear in the effective strength $2q-1$ ($q$=reliability, $r{=}0.98$), so measured reliability is
  the effective pin weight of \eqref{eq:wbalance}. (b)~The leverage$-$random placement gap vanishes as agents
  become more grounded (content-free $\to$ partial $\to$ full evidence), confirming that placement matters
  only in the majority-copy regime. Swarms of $N{=}25$ on a Barab\'asi--Albert graph; panel (a) $k{=}f{=}6$,
  $40$ trials; panel (b) $k{=}4$, $f{=}6$, $12$ questions.}
  \label{fig:swarm_mech}
\end{figure}

\begin{figure}[t]
  \centering
  \includegraphics[width=0.95\linewidth]{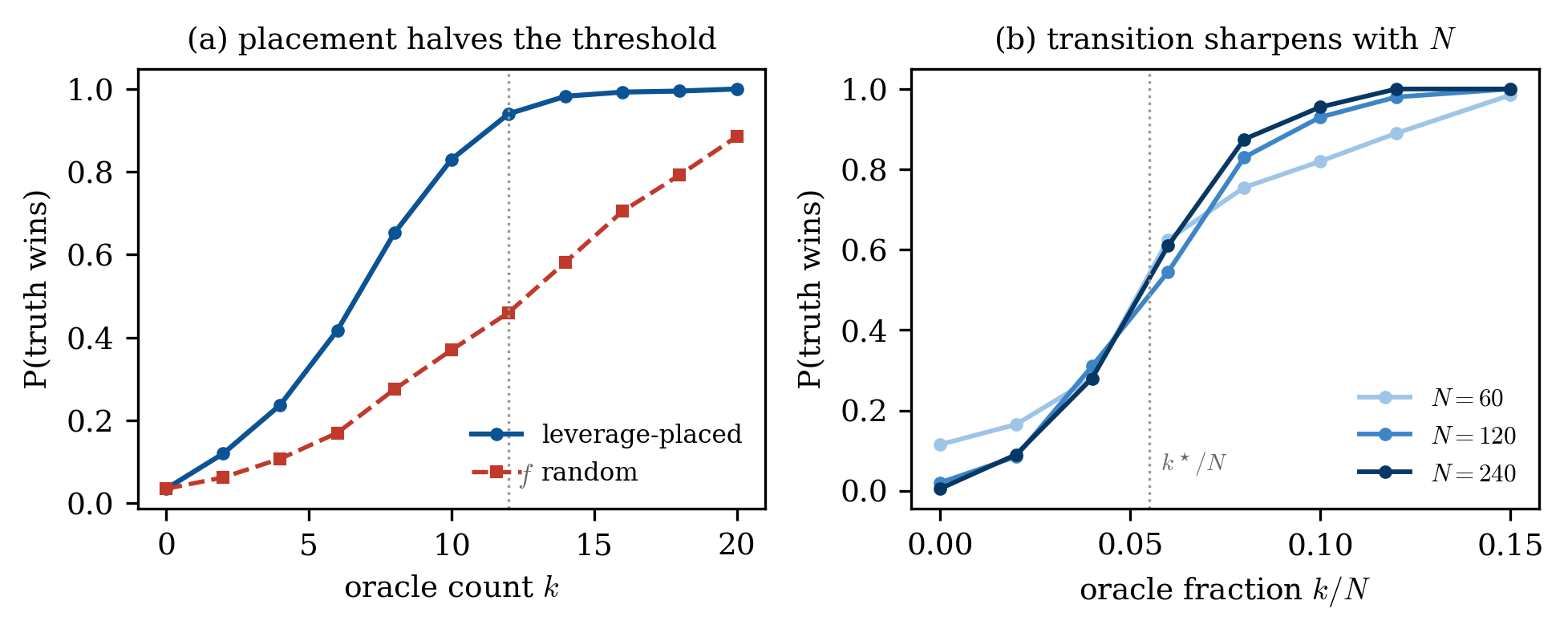}
  \caption{\textbf{The budget--correctness phase transition (Stage~B).} Majority-cascade model: $f$
  false-pinned seeds versus $k$ truth-pinned oracles on an Erd\H{o}s--R\'enyi swarm. (a)~$\mathrm
  P(\text{truth wins})$ jumps sharply with the oracle count, and leverage placement reaches certainty at
  roughly half the oracles random placement needs ($k^\star\approx f\bar d/d_{\max}$ versus $f$).
  (b)~against the oracle fraction $k/N$ the curves steepen with $N$ at an $N$-independent critical
  fraction, a genuine thermodynamic-limit threshold, not a smooth crossover. Estimated over $400$ trials with
  common random numbers across $k$ (each scenario's graph, seeds, and start reused, only the corrector set
  varying), so every point is unbiased and the curves are smooth.}
  \label{fig:phase}
\end{figure}

\begin{proposition}[exact correctness threshold on the complete graph]\label{prop:stageB}
Consider $K_N$ with $k$ agents pinned to truth $(+1)$, $f$ pinned to a falsehood $(-1)$, and the rest free,
each free agent updating synchronously to the sign of its neighbour-sum with reliability
$p_{\mathrm r}\in(0,1)$ from a balanced start. Then truth wins the final consensus with high probability when
$k-f\gg\sqrt N$ and loses when $f-k\gg\sqrt N$. The transition has width $O(\sqrt N)$ in the count, so the
critical fraction $\rho_R=\rho_F$, that is $k>f$, is sharp as $N\to\infty$ (Fig.~\ref{fig:phase}b). On a
general graph, weighting each pin by its degree extends the balance to $\sum_{i\in R}d_i>\sum_{j\in F}d_j$.
\end{proposition}
\begin{proof}
Every free agent sees essentially the same neighbour-sum, equal to the global signed total $S=\sum_j x_j$
(the ``field'') up to an $O(1)$ correction, so all free agents flip the same way, to
$\operatorname{sign}(S)$, and the system reduces to a single scalar update. From the balanced start the
initial field is $S_0=(k-f)+O(\sqrt N)$, the $O(\sqrt N)$ coming from the free spins and the reliability
glitches, and the update locks to $\operatorname{sign}(S_0)$. The stated thresholds follow.
\end{proof}

Beyond the complete graph, the same threshold holds rigorously once the swarm mixes well.

\begin{theorem}[threshold on dense random graphs]\label{thm:stageB}
Run the majority cascade of Proposition~\ref{prop:stageB} on the Erd\H{o}s--R\'enyi graph $G(N,p)$ with
$pN\ge C\log N$, at reliability $p_{\mathrm r}\in(0,1)$ from a balanced start, with $C,C'$ below absolute
constants. Then truth wins the consensus with high probability once the oracles beat the seeds by a vanishing
margin,
\[
  k-f\ \ge\ C'\sqrt{N\log N/p},
\]
and the falsehood wins when the margin has the opposite sign. When $pN=\omega(\log N)$ this margin is $o(N)$,
so the critical fraction is simply $\rho_R=\rho_F$: any fixed surplus of oracles over seeds wins as
$N\to\infty$, with transition width $O(\sqrt N)$ in the count. At the sparse boundary $pN=\Theta(\log N)$ the
margin is $\Theta(N)$, so there the balance is sharp only up to a constant. On a graph with heterogeneous
degrees the same balance holds by degree, $D_R>D_F$.
\end{theorem}
\begin{proof}
See \ref{app:stageB}.
\end{proof}

Sparsity alone does not break the threshold (a sparse expander still mixes); modularity does. When
false seeds cluster in one part of the network, the cascade is confined and resolves locally.

\begin{proposition}[community threshold on modular graphs]\label{prop:modular}
Partition the graph into internally dense communities $C_1,\dots,C_K$ and take the decoupling limit of
vanishing inter-community coupling. Then each community $C_\ell$ reaches truth with high probability if and
only if its internal influence balance holds,
\[
  \sum_{i\in R\cap C_\ell}d_i\ >\ \sum_{j\in F\cap C_\ell}d_j ,
\]
and the global consensus is the size-weighted majority of the communities.
\end{proposition}

Each $C_\ell$ is itself dense and ER-like, so this is Theorem~\ref{thm:stageB} applied community by community.
Its consequence for placement is a covering requirement: the budget must win every seeded community on its
own, rather than maximize one global influence score, so the simple dense-graph rule $D_R>D_F$ no longer
suffices under clustered seeding. For weak but nonzero coupling this per-community resolution is observed in
simulation on a stochastic block model.

\begin{proposition}[hysteresis on sparse random-regular graphs]\label{prop:sparse}
On a sparse random $d$-regular graph there is a threshold $\rho_R^\star(\rho_F)>\rho_F$ such that the
truth-consensus is reached with high
probability
\begin{itemize}\itemsep2pt
\item[\textup{(i)}] from a balanced start, if and only if $\rho_R>\rho_F$; and
\item[\textup{(ii)}] from an established-falsehood start, if and only if $\rho_R>\rho_R^\star(\rho_F)$.
\end{itemize}
Since $\rho_R^\star(\rho_F)>\rho_F$, for $\rho_R\in(\rho_F,\rho_R^\star(\rho_F))$ the outcome depends on the
initial configuration; this gap is the hysteresis.
\end{proposition}

On a sparse graph the field no longer concentrates, so the two starts have distinct mechanisms. Prevention is
tracked by a \emph{cavity recursion} for the fraction
$q$ of free agents that settle at $+1$,
\[
  q=\rho_R+(1-\rho_R-\rho_F)\big[p_{\mathrm r}\Pr(\mathrm{Bin}(d-1,q)>\tfrac{d-1}2)+\tfrac{1-p_{\mathrm r}}2\big].
\]
Near $\rho_R=\rho_F$ and for reliable agents ($p_{\mathrm r}$ near $1$) it has three fixed points, so the
dynamics are bistable and the outcome depends on the starting majority; glitches erode the metastable middle
branch, and the bistable window survives only for $\rho\lesssim0.005$ at $p_{\mathrm r}=0.9$.

Dislodging is instead a monotone \emph{bootstrap percolation}: a free node flips to $+1$ once at least
$\lceil d/2\rceil$ of its neighbours are $+1$, and stays. The activated fraction $z$ obeys
\[
  z=\rho_R+(1-\rho_R-\rho_F)\,\Pr\!\big(\mathrm{Bin}(d-1,z)\ge\lceil d/2\rceil\big),
\]
whose saddle-node is the first-order threshold $\rho_R^\star$: at $d=5$ it is $\approx0.28$ as $\rho_F\to0$,
rising to $\approx0.32$ at $\rho_F=0.1$, several times $\rho_F$ throughout. The saddle-node is rigorous
\cite{BaloghPittel}, and simulation confirms it within about two percentage points for $d=4,\dots,8$. This
prevent-versus-dislodge hysteresis is the reversible bootstrap percolation of fake-news spreading and
fact-checking \cite{DiMuro}, here specialized to a costed grounded-Laplacian swarm with a placement objective.

\begin{remark}[coherence is the cascade linearization]\label{rmk:unif}
Soften the majority update to $x_i(t{+}1)=x_i(t)+\Phi\big((g-M(R)x)_i\big)$ with an odd sigmoid $\Phi$ whose
slope $\eta=\Phi'(0)$ encodes reliability, and forcing $g=\kappa\mathbf 1+W_R\mathbf 1$. Since $M(R)\mathbf 1=g$, the
all-truth state is an exact fixed point, and its Jacobian is exactly $I-\eta M(R)$. Hence the near-consensus
relaxation and steady-state tracking are governed by the very grounded Laplacian $M(R)$ of Stage~A: the
stationary error trace under isotropic per-step glitch noise is $\propto\tr M(R)^{-1}=H(R)$ to leading order
in $\eta$ (distinct from the forcing-noise mean-squared error $\sigma^2\tr M(R)^{-2}$ of \eqref{eq:H}, both minimized by
leverage placement), and the marginal gain is
the Sherman--Morrison leverage \eqref{eq:marg}. So the Stage-A coherence and its greedy leverage placement
are the linearization of the cascade about truth. The correspondence is exact only near consensus:
the coherence optimum $\arg\min_R\tr M(R)^{-1}$ and the nonlinear basin-margin optimum
$\arg\max_R\lambda_{\min}(M(R))$ are different spectral functionals of the same $M(R)$ and generically
disagree (in the large majority, ${>}90\%$, of random graphs), while the hard-sign thresholds of this section are non-spectral.
Leverage/greedy placement is therefore provably optimal for the coherence objective in the linear limit, not
for the global cascade threshold.
\end{remark}

Adding the verification delay of \cite{ItkinDelay} couples
this budget axis to the dose--delay stability axis; both extensions are deferred to keep the present
frontier closed-form.

\section{Conclusion}\label{sec:conc}
We framed multi-agent correctness as a budgeting problem: how much to spend on strong correctors, and
where to place them, so a weak swarm reaches truth. Modelling the corrected swarm as a grounded-Laplacian
consensus, we showed the truth-tracking error stays submodular under heterogeneous cost-coupled pins. A
cost-benefit greedy is therefore within $1-1/e$ of optimal and yields a closed-form budget--correctness
frontier $B^\star$, with a minimal oracle count $k^\star$ as its unit-cost limit. Whether that budget is
best spent on a few strong correctors or many medium ones is decided by a single scalar, the curvature of
the cost--quality law. Measuring it on the Qwen3 ladder (replicated on a second family) finds the law
concave for factual and math verification, saturating by about $2$--$4$B, so spreading wins and the
``escalate to one strong model'' default is suboptimal there; an emergent code-tracing task instead rises
convexly, so the verdict is genuinely task-dependent and both regimes occur. A high-probability
variant turns the frontier into a sharp phase transition whose threshold is an influence balance, exact on
the complete graph.

\section{Limitations}\label{sec:lim}
The analysis is linear, a grounded-Laplacian consensus, and the verifier is an oracle-but-costly pin, so
content-level effects such as claim provenance and partial corrections are out of scope. The frontier
$B^\star$ is known beyond the complete graph, but only in the mean. It is closed-form on the complete
graph. On random $d$-regular graphs it is the cavity fixed point of \ref{app:sparse}, exact in the
replica-symmetric limit and validated numerically. On stochastic-block-model graphs it is the block
cavity of \ref{app:sbm}, for arbitrary $K\times K$ mixing. On directed swarms it is the
reciprocated-feedback cavity of \ref{app:directed}, in which only mutual trust carries feedback. The
submodular placement guarantee does not extend to the directed case, since its proof uses the symmetry of
$M$. The Stage-B threshold is exact on the complete graph, rigorous on dense expanders
(Theorem~\ref{thm:stageB}), and per-community on modular graphs (Proposition~\ref{prop:modular}); on
unstructured sparse graphs it is hysteretic, the dislodge threshold being a first-order
bootstrap-percolation transition (Proposition~\ref{prop:sparse}).

The single-family measurement establishes the sign of the
curvature (concave) at $n{=}200$, but its scope is limited: one model family, parameter count as the cost
proxy, and a strength that is itself a proxy for the abstract pin. The concave sign is at least robust to
the accuracy-to-strength link: it survives the identity and logit transforms, and because the
saturation is per-item (not only in the mean) it does not invert under any monotone increasing
reparametrization of the quality axis we tested, including adversarial ones. The saturation location is not
fixed: a difficulty sweep (Section~\ref{sec:exp}) moves it from ${\sim}1.7$B on easy distractors to
${\sim}4$B on the subtlest, so we report it as task-dependent rather than universal. A sweep past $14$B
would pin the high end, and a dollar or latency cost axis changes nothing in the affine regime, where it
agrees with parameter count (Remark~\ref{rmk:axis}). The verification delay of
\cite{ItkinDelay} is left out to keep the frontier closed-form.

The real-reasoning-corrector confirmation
(Section~\ref{sec:exp}) is a single configuration: one corrector model (Qwen3-4B), one hallucination-hardness
band, $30$ question draws, one $N{=}25$ Barab\'asi--Albert realization. It shows the leverage-placement
advantage survives emergent, difficulty-correlated corrector errors, but does not sweep corrector strength,
task difficulty, or graph ensemble, and its leverage curve reaches parity rather than certainty.

\paragraph{Reproducibility.} All code and data are released at
\url{https://github.com/YehudaItkin/budgeted-oracle-placement} and seeded. This covers the
measurement, statistics, and figure scripts. It covers the per-item $2$-AFC scores for the Qwen3 ladder
(factual, math, and code tasks) and the Gemma-4 ladder (factual). It covers the real-swarm placement,
fallible-corrector, grounding-sweep, and real-reasoning-corrector runs. And it covers the placement,
Stage-B, and sparse-graph, SBM, and directed-graph cavity simulations.

\appendix
\renewcommand{\thesection}{Appendix~\Alph{section}}
\allowdisplaybreaks
\linespread{1.15}\selectfont
\setlength{\parskip}{2pt}

\clearpage
\section{Proof of Theorem~\ref{thm:submod} (submodularity under heterogeneous pins)}\label{app:submod}
Throughout, fix $\kappa>0$ and write
\[
  M(R)=L+\kappa I+\sum_{i\in R}w_i\,e_ie_i^\top,\qquad H(R)=\tr M(R)^{-1},
\]
where $e_i$ is the $i$-th standard basis vector, so that $e_ie_i^\top$ has a single nonzero entry, a $1$ in
position $(i,i)$. Since $L\succeq0$ and $\kappa>0$, the matrix $M(R)$ is symmetric positive definite and
$H(R)$ is well defined. We prove the three claims of the theorem in turn: the marginal-gain formula, then
monotonicity, then submodularity.

\emph{Marginal gain.} Placing an oracle at node $i$ adds $w_ie_ie_i^\top$ to $M(R)$, a rank-one update.
Abbreviate $M=M(R)$. The Sherman--Morrison identity inverts such an update in closed form,
\[
  \big(M+w_i\,e_ie_i^\top\big)^{-1}
   =M^{-1}-\frac{w_i\,M^{-1}e_ie_i^\top M^{-1}}{1+w_i\,e_i^\top M^{-1}e_i}.
\]
Take the trace of both sides. The trace of the correction is
$\tr\!\big(M^{-1}e_ie_i^\top M^{-1}\big)=e_i^\top M^{-2}e_i=\|M^{-1}e_i\|^2$, so the reduction in coherence from
adding the oracle is
\[
  \Delta_i(R)=H(R)-H(R\cup\{i\})
   =\frac{w_i\,\|M^{-1}e_i\|^2}{1+w_i\,e_i^\top M^{-1}e_i}\ \ge\ 0,
\]
which is \eqref{eq:marg}. Numerator and denominator are both positive, so every oracle strictly lowers $H$:
the gain is the squared norm of the $i$-th column of $M^{-1}$, damped by the local diagonal
$e_i^\top M^{-1}e_i$.

\emph{Monotonicity.} Adding an oracle only increases $M$ in the Loewner order,
\[
  M(R\cup\{i\})=M(R)+w_i\,e_ie_i^\top\ \succeq\ M(R),
\]
because $w_ie_ie_i^\top\succeq0$. Inversion is order-reversing on positive definite matrices, so
$M(R\cup\{i\})^{-1}\preceq M(R)^{-1}$, and taking traces gives $H(R\cup\{i\})\le H(R)$. Hence
$\rho(R)=H(\varnothing)-H(R)$ is non-decreasing in $R$, with $\rho(\varnothing)=0$.

\emph{Submodularity.} Regard $H$ as a smooth function of the strength vector $w=(w_1,\dots,w_N)\ge0$, an
oracle being present at $i$ exactly when $w_i>0$. The derivative of a matrix inverse is
\[
  \frac{\partial M^{-1}}{\partial w_j}
   =-M^{-1}\frac{\partial M}{\partial w_j}M^{-1}
   =-M^{-1}e_je_j^\top M^{-1}.
\]
Differentiate $H=\tr M^{-1}$ once, using $\tr(M^{-1}e_ie_i^\top M^{-1})=(M^{-2})_{ii}$:
\[
  \frac{\partial H}{\partial w_i}
   =\tr\!\Big(\frac{\partial M^{-1}}{\partial w_i}\Big)
   =-(M^{-2})_{ii}.
\]
Differentiate a second time in $w_j$. The factor $M^{-2}=M^{-1}M^{-1}$ has two copies of $M^{-1}$ and $w_j$
acts on each; the two contributions are equal by symmetry, so
\[
  \frac{\partial^2 H}{\partial w_i\,\partial w_j}
   =-\frac{\partial (M^{-2})_{ii}}{\partial w_j}
   =2\,(M^{-1})_{ij}\,(M^{-2})_{ij}.
\]
The sign of this cross-derivative is fixed by one structural fact about $M$. Its off-diagonal entries are
$-A_{ij}\le0$, and it is positive definite, so $M(R)$ is a symmetric \emph{M-matrix}. The inverse of an
M-matrix is entrywise nonnegative, $M^{-1}\ge0$, and therefore $M^{-2}=M^{-1}M^{-1}\ge0$ as well. Both factors
in $2\,(M^{-1})_{ij}\,(M^{-2})_{ij}$ are then nonnegative, so
\[
  \frac{\partial^2 H}{\partial w_i\,\partial w_j}\ \ge\ 0
  \qquad\text{for all }i,j\text{ and all }w\ge0 .
\]
Nonnegative mixed second derivatives are exactly the statement that each marginal gain shrinks as the other
strengths grow. Write the gain at $i$ as the line integral
\[
  \Delta_i(R)=\int_0^{w_i}\Big(-\frac{\partial H}{\partial w_i}\Big)\,\mathrm dw_i
             =\int_0^{w_i}(M^{-2})_{ii}\,\mathrm dw_i .
\]
The integrand $(M^{-2})_{ii}$ is non-increasing in every other $w_j$, so $\Delta_i(R)\ge\Delta_i(S)$ whenever
$R\subseteq S$. This is submodularity of $\rho=-H$.

The Loewner order alone does not deliver this last step. Monotonicity of $M\mapsto M^{-1}$ in the Loewner
sense does not control the \emph{entrywise} signs of $M^{-1}$ for a general positive definite $M$, and it is
those signs that make the cross-derivative nonnegative. A numerical sweep confirms both monotonicity and
submodularity across the diminishing-returns, linear, and uniform weight laws.

\section{Proof of Theorem~\ref{thm:knapsack}}\label{app:knapsack}
By Theorem~\ref{thm:submod} the reduction $\rho$ is monotone submodular with $\rho(\varnothing)=0$, and the
cost $\sum_{i\in R}c_i$ is additive in the chosen oracles. Maximizing a monotone submodular function under a
single budget constraint is exactly the setting of \cite{Sviridenko}. Its algorithm enumerates every seed set
of at most three oracles, extends each seed by repeatedly adding the oracle of largest gain-to-cost ratio
within budget, and returns the best placement found; the guarantee is attained on the branch whose seed is
the three highest-value oracles of the optimum. Call its output $R_g$ and let $R^\star_B$ be the optimum
within budget $B$. Then
\[
\rho(R_g)\ \ge\ (1-1/e)\,\rho(R^\star_B).
\]

The cost-benefit greedy of Algorithm~\ref{alg:greedy} is this procedure run with the CELF lazy-evaluation
rule \cite{Leskovec}. Lazy evaluation changes only how the inner greedy step is computed, not which oracle
it selects, so the $(1-1/e)$ bound is untouched. It uses submodularity (marginal gains only shrink as $R$
grows) to skip recomputing gains that cannot be the current maximum, so a greedy step recomputes at most
$O(|\Oc|)$ gains, and typically far fewer, each an $O(N)$ read of one column of the maintained resolvent
$M(R)^{-1}$, with $O(\log|\Oc|)$ priority-queue overhead per recomputation. Once the winner is committed,
$M(R)^{-1}$ is advanced by a single Sherman--Morrison rank-one update, since adding one oracle changes $M(R)$
in a single diagonal entry.

\section{The complete-graph frontier and the curvature criterion}\label{app:frontier}
\emph{Closed form \eqref{eq:HKN}.} On the complete graph $K_N$ the Laplacian is $L=NI-\mathbf 1\mathbf 1^\top$,
where $\mathbf 1$ is the all-ones vector, so a placement $R$ gives
\[
  M(R)=(N+\kappa)I-\mathbf 1\mathbf 1^\top+\diag\!\big(w_i\,\mathds 1[i\in R]\big)=\Lambda-\mathbf 1\mathbf 1^\top,
\]
with $\Lambda=\diag(\delta_i)$ diagonal and $\delta_i=N+\kappa+w_i\,\mathds 1[i\in R]$, matching
Proposition~\ref{prop:KN}. This is a rank-one downdate of the diagonal $\Lambda$, so Sherman--Morrison inverts
it in closed form,
\[
  M(R)^{-1}=\Lambda^{-1}+\frac{\Lambda^{-1}\mathbf 1\mathbf 1^\top\Lambda^{-1}}{1-\mathbf 1^\top\Lambda^{-1}\mathbf 1}.
\]
Take the trace and set $S_1=\mathbf 1^\top\Lambda^{-1}\mathbf 1=\sum_i\delta_i^{-1}$ and
$S_2=\mathbf 1^\top\Lambda^{-2}\mathbf 1=\sum_i\delta_i^{-2}$; using
$\tr(\Lambda^{-1}\mathbf 1\mathbf 1^\top\Lambda^{-1})=\mathbf 1^\top\Lambda^{-2}\mathbf 1=S_2$ gives
\[
  H(R)=\tr M(R)^{-1}=S_1+\frac{S_2}{1-S_1},
\]
which is \eqref{eq:HKN}. The matrix is positive definite exactly when $1-S_1>0$, and $\kappa>0$ secures this,
since $\delta_i\ge N+\kappa$ forces $S_1\le N/(N+\kappa)<1$. The closed form matches the direct inverse to
machine precision (relative error ${\sim}10^{-15}$ over random heterogeneous placements). In the homogeneous
case $H(k)$, with $k$ equal pins, is strictly decreasing in $k$, so the minimal count $k^\star(\varepsilon)$
is obtained by inverting $H(k)=\varepsilon$.

\emph{Curvature criterion.} \emph{Equal split at fixed $m$.} Write $\theta:=N+\kappa$ and, for an oracle of
cost $c$, let
\[
  g(c)=\frac{1}{\theta+w(c)}
\]
be the diagonal entry it contributes to $\Lambda^{-1}$. Fix the number of oracles at $m$ and let them carry
costs $c_1,\dots,c_m$ with $\sum_i c_i=B$. The $N-m$ unpinned nodes each contribute $1/\theta$ to $S_1$ and
$1/\theta^2$ to $S_2$, constants at fixed $m$, so the split enters $H$ only through the oracle sums
\[
  S_1=\frac{N-m}{\theta}+\sum_{i=1}^m g(c_i),\qquad S_2=\frac{N-m}{\theta^2}+\sum_{i=1}^m g(c_i)^2 .
\]
Two facts make the equal split optimal. First, $g$ and $g^2$ are convex in $c$: the outer map
$s\mapsto1/(\theta+s)$ is convex and decreasing, $w$ is concave by Assumption~\ref{ass:wc}, so the composition
$g$ is convex, and $(g^2)''=2(g')^2+2g\,g''\ge0$ makes $g^2$ convex as well. Second, $H=S_1+S_2/(1-S_1)$ is
increasing in each of $S_1,S_2$, since
\[
  \frac{\partial H}{\partial S_1}=1+\frac{S_2}{(1-S_1)^2}>0,\qquad
  \frac{\partial H}{\partial S_2}=\frac{1}{1-S_1}>0 .
\]
By convexity and Jensen's inequality, at fixed budget $\sum_i c_i=B$,
\[
  \sum_{i=1}^m g(c_i)\ \ge\ m\,g\!\Big(\frac{B}{m}\Big),\qquad
  \sum_{i=1}^m g(c_i)^2\ \ge\ m\,g\!\Big(\frac{B}{m}\Big)^2,
\]
with equality only when all $c_i$ equal $B/m$. The equal split thus minimizes both $S_1$ and $S_2$, and since
$H$ increases in each, it minimizes $H$. Every unequal split is strictly worse.

\emph{Choice of $m$.} Now vary the number of oracles, each at equal cost $c=B/m$. The leading-order total
$m\,v(B/m)$ has derivative with the sign of $v(c)-c\,v'(c)=-c^2\,\tfrac{\mathrm d}{\mathrm dc}(v(c)/c)$ at
$c=B/m$. The marginal gain $\Delta$ is increasing and concave in strength (its derivative
$\|M^{-1}e_i\|^2/(1+w\,e_i^\top M^{-1}e_i)^2$ is positive and decreasing), so $v=\Delta\circ w$ inherits
concavity from $w$. Then $v(c)/c$ decreases and spreading (large $m$) is favoured. This leading-order
heuristic agrees with the exact $H$ of \eqref{eq:HKN}: the equal-split $H(m)$ strictly decreases in $m$ under
concave $w$, and increases only where $w$ is convex strongly enough to satisfy \eqref{eq:concentrate}. Plain
convexity keeps spreading optimal. The exact statement is Theorem~\ref{thm:mmono} and Remark~\ref{rem:sharp}.

\emph{Graph leverage.} On non-vertex-transitive graphs the optimum equalizes the per-dollar ratios
\eqref{eq:perdollar}. Equivalent nodes still receive equal strength; non-equivalent nodes receive
leverage-weighted amounts. Graph leverage thus shifts which nodes enter and how much each gets, but not the
concave/sharply-convex dichotomy. A numerical sweep confirms this. On a star the continuous optimum spreads
equally over the symmetric leaves under concave $w$. Over Erd\H{o}s--R\'enyi graphs concave $w$ spreads
across ${\approx}80\%$ of nodes, while a sharply convex law ($w(c)=c^2$, satisfying \eqref{eq:concentrate})
concentrates on the few highest-leverage ones. Throughout, $k^\star(\varepsilon)$ stays monotone.

\emph{Minimal count.} For the coherence criterion the minimal count scales as a constant \emph{fraction} of
$N$, not as the worst-case $\Theta(N/d)$ of $\lambda_{\min}$ \cite{PiraniSundaram}. This fraction is
$N$-independent (${\approx}0.3$ at $d{=}6$, well above $1/d$; Fig.~\ref{fig:frontier}b), and its exact value
is left open.

\section{Global optimality of the equal spread (Theorem~\ref{thm:curv})}\label{app:spread}
Throughout, set $\theta:=N+\kappa>0$, and for a cost $c\ge0$ write
\[
  g(c):=\frac{1}{\theta+w(c)},\qquad \psi(c):=g(c)^2 .
\]
On $K_N$ the exact coherence \eqref{eq:HKN} is then $H=S_1+S_2/(1-S_1)$ with $S_1=\sum_i g(c_i)$ and
$S_2=\sum_i\psi(c_i)$. We minimise $H$ over the budget simplex
$\Delta_B=\{c\in\R_{\ge0}^N:\sum_i c_i=B\}$. Since $\theta>N$, each $g(c_i)\le1/\theta$, so $S_1<1$ on all of
$\Delta_B$ and the denominator $1-S_1$ stays positive.

\begin{lemma}[convexity of the cost--coherence primitives]\label{lem:gpsi}
Under Assumption~\ref{ass:wc}, $g$ and $\psi$ are strictly decreasing and convex on $[0,\infty)$, with
\begin{equation}\label{eq:gpp}
  g''(c)=\frac{-w''(c)(\theta+w(c))+2w'(c)^2}{(\theta+w(c))^3},\qquad
  \psi''(c)=\frac{-2w''(c)(\theta+w(c))+6w'(c)^2}{(\theta+w(c))^4}.
\end{equation}
\end{lemma}
\begin{proof}
Differentiate $g(c)=(\theta+w(c))^{-1}$ by the chain rule:
\[
  g'(c)=-\frac{w'(c)}{(\theta+w(c))^2},\qquad
  g''(c)=-\frac{w''(c)}{(\theta+w(c))^2}+\frac{2w'(c)^2}{(\theta+w(c))^3},
\]
and collecting over the common denominator gives the stated $g''$. For $\psi=g^2$,
$\psi'=2g\,g'=-2w'/(\theta+w)^3$, and one more differentiation gives the stated $\psi''$. Now read off the
signs. Since $w'\ge0$, both $g'\le0$ and $\psi'\le0$, so $g$ and $\psi$ are decreasing. In each
second-derivative numerator in \eqref{eq:gpp} the term $-w''(\theta+w)$ is $\ge0$ (because $w''\le0$ and
$\theta+w>0$) and the $w'^2$ term is $\ge0$, so $g''\ge0$ and $\psi''\ge0$, i.e.\ both are convex.
\end{proof}

\begin{lemma}[pairwise balancing]\label{lem:balance}
Fix all agents except a pair $\{a,b\}$ sharing a budget $s=c_a+c_b$, and let $r:=\sum_{i\neq a,b}g(c_i)<1$,
$q:=\sum_{i\neq a,b}\psi(c_i)$ be the frozen contributions. Writing $c_a=\tfrac s2+t$, $c_b=\tfrac s2-t$, the
map $F(t):=H$ is even, convex, and uniquely minimised at $t=0$; and $S_1$ is non-increasing as $t\to0$,
so admissibility is preserved.
\end{lemma}
\begin{proof}
Let $u(t)=g(\tfrac s2+t)+g(\tfrac s2-t)$ and $v(t)=\psi(\tfrac s2+t)+\psi(\tfrac s2-t)$, so that
\[
  H=(u+r)+\frac{v+q}{1-(u+r)}=:\Phi(u,v).
\]
With $D:=1-(u+r)>0$, the partial derivatives of $\Phi$ are
\[
  \Phi_u=1+\frac{q+v}{D^2}>0,\qquad \Phi_v=\frac1D>0,\qquad
  \Phi_{uu}=\frac{2(q+v)}{D^3}\ge0,\qquad \Phi_{vv}=0,\qquad \Phi_{uv}=\frac1{D^2}>0.
\]
By Lemma~\ref{lem:gpsi}, $u$ and $v$ are even and convex, so $u(0)=\min u$, $v(0)=\min v$, and $F'(0)=0$. For
$t>0$, the monotonicity of $g',\psi'$ gives $u'(t),v'(t)>0$, hence $u'v'\ge0$. By the second-order chain rule
for the composition $F(t)=\Phi(u(t),v(t))$,
\[
  F''(t)=\Phi_{uu}\,(u')^2+2\,\Phi_{uv}\,u'v'+\Phi_{vv}\,(v')^2+\Phi_u\,u''+\Phi_v\,v'' .
\]
Every term on the right is nonnegative: $\Phi_u,\Phi_v,\Phi_{uu},\Phi_{uv}\ge0$ from the partials above,
$u'',v''\ge0$ by convexity (Lemma~\ref{lem:gpsi}), and $(u')^2,(v')^2,u'v'\ge0$. Thus $F''(t)\ge0$, so $F$ is
convex with a critical point at $t=0$, and $t=0$ is its unique minimiser. Convexity of $g$ makes $S_1=u+r$
smallest there as well.
\end{proof}

\begin{theorem}[exact global spreading verdict]\label{thm:global}
Under Assumption~\ref{ass:wc}, the exact coherence $H$ of \eqref{eq:HKN} attains its unique minimum over
$\Delta_B$ at the maximal equal spread $c_i^\star=B/N$ for all $i$; hence $\rho(R)=H(\varnothing)-H(R)$ is
globally maximised there.
\end{theorem}
\begin{proof}
$H$ is continuous on the compact $\Delta_B$, so it attains a minimum at some $c^\star$. If two coordinates
of $c^\star$ differed, Lemma~\ref{lem:balance} would strictly decrease $H$ inside $\Delta_B$, contradicting
minimality; hence $c^\star$ is constant, $c_i^\star=B/N$, and strict convexity of each balancing step gives
uniqueness.
\end{proof}

\begin{theorem}[exact $m$-monotonicity]\label{thm:mmono}
For $m$ oracles of cost $c=B/m$ (others at $0$), with $g_0:=g(0)=1/\theta$ (using $w(0)=0$: a zero-cost oracle adds no strength),
\begin{equation}\label{eq:dHdm}
  \frac{\mathrm dH}{\mathrm dm}=\alpha\big(g(c)-c\,g'(c)-g_0\big)+\beta\big(\psi(c)-c\,\psi'(c)-\psi(0)\big),
  \quad \alpha=1+\tfrac{S_2}{(1-S_1)^2}>0,\ \beta=\tfrac1{1-S_1}>0.
\end{equation}
Under Assumption~\ref{ass:wc} the two bracketed terms are $\le0$ (they vanish at $c=0$ and have
$c$-derivatives $-c\,g''\le0$, $-c\,\psi''\le0$), so $\mathrm dH/\mathrm dm\le0$, strictly for strictly
concave $w$: $H(m)$ is strictly decreasing and the optimum is the maximal spread $m^\star=N$, consistent with
Theorem~\ref{thm:global}.
\end{theorem}
\begin{proof}
Treat $m$ as continuous and put $c=B/m$. Separate each sum into the $N$-node baseline and the oracle
correction, using $g_0=g(0)=1/\theta$ and $\psi(0)=g_0^2$:
\[
  S_1=Ng_0+\big(P(m)-mg_0\big),\qquad S_2=Ng_0^2+\big(Q(m)-mg_0^2\big),
\]
where $P(m)=m\,g(B/m)$ and $Q(m)=m\,\psi(B/m)$ carry the $m$ oracles. Since $\mathrm dc/\mathrm dm=-B/m^2=-c/m$,
\[
  P'(m)=g(c)+m\,g'(c)\,\frac{\mathrm dc}{\mathrm dm}=g(c)-c\,g'(c),\qquad Q'(m)=\psi(c)-c\,\psi'(c).
\]
Differentiating $H=S_1+S_2/(1-S_1)$ gives $\mathrm dH/\mathrm dm=\alpha\,\mathrm dS_1/\mathrm dm+\beta\,\mathrm
dS_2/\mathrm dm$ with $\alpha,\beta$ as in \eqref{eq:dHdm}, and $\mathrm dS_1/\mathrm dm=P'(m)-g_0$, $\mathrm
dS_2/\mathrm dm=Q'(m)-g_0^2$; substituting yields \eqref{eq:dHdm}. Each bracket vanishes at $c=0$ and has
$c$-derivative
\[
  \frac{\mathrm d}{\mathrm dc}\big(g(c)-c\,g'(c)\big)=-c\,g''(c)\le0
\]
(and $-c\,\psi''\le0$ likewise), by $g'',\psi''\ge0$ from Lemma~\ref{lem:gpsi}. So both brackets are $\le0$ for
$c\ge0$, giving $\mathrm dH/\mathrm dm\le0$. Under strictly concave $w$, $g''>0$ on a set of positive measure,
so the inequality is strict.
\end{proof}

\begin{remark}[sharpness: convexity is not enough for concentration]\label{rem:sharp}
The sign in \eqref{eq:dHdm} is governed by the numerators $-w''(\theta+w)+2(w')^2$ and $-w''(\theta+w)+3(w')^2$ of
\eqref{eq:gpp}. Concavity of $w$ makes both nonnegative (spreading). The reverse verdict
$\mathrm dH/\mathrm dm\ge0$ (concentrate on few oracles) requires the stronger pointwise condition
$(\theta+w)\,w''\ge3(w')^2$ of \eqref{eq:concentrate}: $w$ must be convex \emph{strongly enough} that its
curvature dominates $(w')^2$. Plain convexity does not suffice: for weakly convex laws the $2(w')^2$
(resp.\ $3(w')^2$) term still wins and spreading remains optimal. A numerical sweep confirms both directions:
over $2\cdot10^4$ random concave laws the equal spread over all $N$ nodes is never beaten, while $140/8000$
merely-convex laws still spread; a global optimizer on the full simplex lands on $c_i=B/N$ to machine
precision under concave $w$ and on a single oracle under $w(c)=c^2$ once the budget is small enough ($B\le\sqrt{\theta/5}$, $\theta=N+\kappa$) that \eqref{eq:concentrate} holds across $[0,B]$; past that the discrete optimum spreads.
\end{remark}

\section{Proof of Theorem~\ref{thm:stageB} (threshold on dense random graphs)}\label{app:stageB}
Let $m_t=\frac1N\sum_j x_j(t)$ be the magnetization. On $G(N,p)$ with $pN\ge C\log N$ the degrees
concentrate, $d_i=pN\,(1+o(1))$ uniformly. Fix a free node $i$; conditioned on the current state its
neighbours are a uniform $p$-random subset, so the neighbour-sum $S_i=\sum_{j\sim i}x_j$ has
$\E[S_i]=pN\,m_t+O(1)$ and, by Hoeffding, $\Pr\!\big(|S_i-\E S_i|\ge t\big)\le 2e^{-t^2/(2d_i)}$. Taking
$t=\tfrac12|\E S_i|$ and a union bound over the $\le N$ free nodes, if
\[
  |m_t|\ \ge\ C_0\sqrt{\frac{\log N}{pN}}
\]
then with probability $1-o(1)$ \emph{every} free node has $\operatorname{sign}(S_i)=\operatorname{sign}(m_t)$,
so all free nodes (bar the glitching $1-p_{\mathrm r}$ minority) update to $\operatorname{sign}(m_t)$.

At the balanced start the free spins sum to $O(\sqrt N)$, so $m_0=(k-f)/N+O(1/\sqrt N)$; because the
degrees concentrate this is the degree-weighted balance,
$\operatorname{sign}(m_0)=\operatorname{sign}(k-f)=\operatorname{sign}(D_R-D_F)$ with
$D_R-D_F=pN\,(k-f)\,(1+o(1))$. When the count margin satisfies $k-f\ge C'\sqrt{N\log N/p}$ (which both
exceeds the $O(\sqrt N)$ start fluctuation and clears the alignment threshold above), $|m_0|$ lies above
$C_0\sqrt{\log N/(pN)}$, so one step drives the free majority to $\operatorname{sign}(k-f)$; the
configuration in which every free agent agrees with the pinned majority is a fixed point, because each free
neighbour-sum then keeps that sign. Hence truth wins when $k-f\ge C'\sqrt{N\log N/p}$ (and fails
symmetrically). When $pN=\omega(\log N)$ this margin is $o(N)$, so the critical fraction $\rho_R=\rho_F$ is
sharp; at the boundary $pN=\Theta(\log N)$ it is $\Theta(N)$ and the fraction is sharp only up to constants. The glitches
flip a $\mathrm{Bin}(N-k-f,\,1-p_{\mathrm r})$ set of free nodes, an $O(\sqrt N)$ fluctuation of the count
margin that, with the balanced-start fluctuation, sets the $O(\sqrt N)$ transition width seen in
simulation. The complete graph is the deterministic $p_{\mathrm r}\to1$, $p=1$ limit, recovering
Proposition~\ref{prop:stageB}.

\section{The sparse-graph frontier via the cavity method}\label{app:sparse}
On a random $d$-regular graph $M=(d+\kappa)I-A+\diag(w_i\mathds 1[i\in R])$ is not covered by the
complete-graph closed form \eqref{eq:HKN}, but the local tree structure makes the coherence exact in the
replica-symmetric limit. The diagonal resolvent $g_i=(M^{-1})_{ii}$ obeys a cavity recursion: writing
$g_{i\to j}$ for the resolvent on the graph with edge $(i,j)$ deleted and $a_i=d+\kappa+w_i$,
\[
  g_{i\to j}=\frac{1}{a_i-\sum_{k\in\partial i\setminus j}g_{k\to i}},
\]
the off-diagonal $-1$ entries squaring to $+1$ in the Schur complement. For a fraction $\rho=k/N$ of nodes
pinned at strength $w$, the cavity resolvent converges to the distributional fixed point
\[
  g\ \stackrel{d}{=}\ \frac{1}{a-\sum_{l=1}^{d-1}g_l},\qquad
  a=\begin{cases}d+\kappa+w&\text{w.p. }\rho,\\[2pt] d+\kappa&\text{w.p. }1-\rho,\end{cases}
\]
with $g_l$ i.i.d.\ copies, and the per-node coherence is
$h(\rho)=H/N=\E\big[(a-\sum_{l=1}^{d}g_l)^{-1}\big]$. Solved by population dynamics, $h(\rho)$ matches the
direct $\tr M^{-1}/N$ on finite random $d$-regular graphs to under $0.5\%$ (at $d=6$, $\kappa=1$, $w=8$:
$h(0)=0.166$, $h(0.3)=0.134$). Inverting $h(\rho)=\varepsilon/N$ gives the minimal pinned fraction
$\rho^\star$, hence the frontier $B^\star=\rho^\star N\,c$. It is $N$-independent: the exact
replica-symmetric sparse-graph analogue of \eqref{eq:HKN}.

\section{Structured sparse graphs: the stochastic block model cavity}\label{app:sbm}
\ref{app:sparse} gives the exact coherence on the random $d$-regular ensemble. Real agent
topologies cluster into communities (teams, tool groups, debate sub-panels) wired densely inside and sparsely
across. We extend the cavity to the \emph{stochastic block model} (SBM). Fix $K$ communities of sizes
$N_\ell$; every node in community $\ell$ has $d_{in}$ intra-community and $d_{out}$ inter-community
neighbours ($d_{in},d_{out}=O(1)$, symmetric mixing), total degree $d=d_{in}+d_{out}$, and a fraction
$\rho_\ell$ of pinned oracles of strength $w$, so $M(R)_{ii}=d+\kappa+w\,\mathds 1[i\in R]$.

A cavity message must remember both its source community and whether it crosses a boundary. Let $g^{in}_\ell$,
$g^{out}_\ell$ be the cavity variances on intra- and inter-community directed edges with source in $\ell$,
and $a_\ell=d+\kappa+w\,B_\ell$ with $B_\ell\sim\mathrm{Bernoulli}(\rho_\ell)$. The leave-one-out elimination
gives the coupled fixed point
\begin{equation}\label{eq:sbm-cav}
  g^{in}_\ell\overset{d}{=}\Big(a_\ell-\!\sum_{r=1}^{d_{in}-1} g^{in}_{\ell,r}-\!\sum_{s=1}^{d_{out}} g^{out}_{\ell'_s,s}\Big)^{-1},
  \quad
  g^{out}_\ell\overset{d}{=}\Big(a_\ell-\!\sum_{r=1}^{d_{in}} g^{in}_{\ell,r}-\!\sum_{s=1}^{d_{out}-1} g^{out}_{\ell'_s,s}\Big)^{-1},
\end{equation}
with i.i.d.\ copies, each $\ell'_s$ drawn uniformly from the other communities. The block coherence is the node
marginal $h_\ell(\rho_\ell)=\E\big[(a_\ell-\sum_{r=1}^{d_{in}} g^{in}_{\ell,r}-\sum_{s=1}^{d_{out}} g^{out}_{\ell'_s,s})^{-1}\big]$,
and since the trace splits over communities,
\begin{equation}\label{eq:sbm-Htot}
  \frac{H(R)}{N}=\frac{\tr M(R)^{-1}}{N}=\sum_{\ell=1}^{K}\frac{N_\ell}{N}\,h_\ell(\rho_\ell).
\end{equation}
Two limits check out. When $d_{out}=d_{in}$ (or $K=1$) both populations coincide and \eqref{eq:sbm-cav}
collapses to the scalar random-$d$-regular cavity of \ref{app:sparse}. As $d_{out}\to0$ the blocks
decouple and \eqref{eq:sbm-Htot} reduces to the additive per-community covering picture of
Proposition~\ref{prop:modular}, with the concave-$w$ spreading law of Theorem~\ref{thm:curv} applied
\emph{within} each block; $d_{out}$ interpolates between one homogeneous swarm and $K$ independent ones.

We solve \eqref{eq:sbm-cav} by population dynamics and compare against the direct block average of
$(M(R)^{-1})_{ii}$ on finite SBMs ($N=900$--$1200$, $5$--$6$ realizations). The agreement is below $1\%$. For
$K{=}3$, $d_{in}{=}6$, $d_{out}{=}2$, $\rho=(0.1,0.4,0.7)$, $w{=}3$, $\kappa{=}0.5$, the block $h$ is
$(0.130,0.117,0.105)$ direct versus $(0.129,0.117,0.105)$ cavity, and the total $H/N$ is $0.1177$ versus
$0.1172$ ($0.4\%$). The gap halves as $N$ doubles ($0.83\%,0.42\%,0.23\%$ at $N=450,900,1500$), so the finite
graph converges to the cavity value. The frontier thus extends exactly, in the replica-symmetric
locally-tree-like regime, from the homogeneous ensemble to community-structured sparse graphs.

\paragraph{General asymmetric mixing.}
The symmetric ensemble above is the case $d_{\ell\ell}=d_{\mathrm{in}}$, $d_{\ell\ell'}=d_{\mathrm{out}}/(K-1)$
of a general \emph{degree matrix} $d_{\ell\ell'}\ge0$, the mean number of neighbours a block-$\ell$ node has
in block $\ell'$; blocks may also carry distinct $\rho_\ell$ and strengths $w_\ell$. A cavity message must now
remember its \emph{ordered} endpoints: let $g_{\ell\to\ell'}$ be the cavity resolvent on a directed edge from
a block-$\ell$ node toward a block-$\ell'$ node (one population per ordered pair, $K^2$ in all). With
\begin{equation}\label{eq:asbm-a}
  a_\ell=\kappa+\sum_{\ell'=1}^{K} d_{\ell\ell'}+w_\ell\,B_\ell,\qquad B_\ell\sim\mathrm{Bernoulli}(\rho_\ell),
\end{equation}
the leave-one-out elimination (the message $\ell\to\ell'$ keeps every incoming neighbour of its source node
except the single $\ell'$-neighbour it answers) gives
\begin{equation}\label{eq:asbm-cav}
  g_{\ell\to\ell'}\overset{d}{=}\Big(a_\ell-\sum_{p=1}^{K}\ \sum_{r=1}^{\,d_{\ell p}-\delta_{p\ell'}} g^{(r)}_{p\to\ell}\Big)^{-1},
\end{equation}
with independent copies from population $g_{p\to\ell}$. The block coherence sums over all incoming neighbours,
\begin{equation}\label{eq:asbm-h}
  h_\ell=\E\Big[\Big(a_\ell-\sum_{p=1}^{K}\ \sum_{r=1}^{\,d_{\ell p}} g^{(r)}_{p\to\ell}\Big)^{-1}\Big],
  \qquad \frac{H(R)}{N}=\sum_{\ell=1}^{K}\frac{N_\ell}{N}\,h_\ell .
\end{equation}
Setting $d_{\ell\ell}=d_{\mathrm{in}}$, $d_{\ell\ell'}=d_{\mathrm{out}}/(K-1)$ collapses the $K-1$
cross-populations with a common target into a single $g^{\mathrm{out}}_\ell$ and recovers
\eqref{eq:sbm-cav}--\eqref{eq:sbm-Htot} exactly. (These block equations sum a node's feedback over all of its
incoming neighbours, which is exact precisely because the support is symmetric; the genuinely directed case is
treated in \ref{app:directed}.) Population dynamics on the $K^2$ populations again matches the
direct block average below $1\%$: for the genuinely asymmetric
$d=\big(\begin{smallmatrix}6&2&1\\4&5&3\\2&3&7\end{smallmatrix}\big)$ with $N_\ell=(600,300,300)$,
$\rho=(0.2,0.5,0.3)$, $w=(3,2,5)$, $\kappa=0.4$, the cavity gives $H/N=0.096$ versus $0.097$ direct
($0.6\%$). The frontier is thus exact, replica-symmetric, for arbitrary $K\times K$ community mixing.

Equations~\eqref{eq:asbm-cav}--\eqref{eq:asbm-h} are stated for the \emph{symmetric-support} ensemble, where
every edge carries influence in both directions (the weights may differ, but $i$ feeds $j$ iff $j$ feeds $i$).
A message then sums its source's incoming neighbours and every incoming neighbour is also an outgoing one.
A genuinely \emph{directed} swarm, where influence can be one-way, is a different object: $M$ becomes
non-symmetric and, as \ref{app:directed} shows, the feedback must be restricted to reciprocated edges
rather than all incoming ones.

\section{Directed swarms: only reciprocated trust carries feedback}\label{app:directed}
So far every edge has been two-way. Real influence is often one-way: an agent may read a tool's output without
the tool reading it back, or a junior agent may copy a senior one it never affects. Model this with a directed
influence graph, $A_{ij}=1$ when agent $j$ feeds agent $i$ ($j\to i$), and the grounded \emph{row}-Laplacian
\begin{equation}\label{eq:dir-M}
  M \;=\; D_{\mathrm{in}} - A + \kappa I + W_R,\qquad
  M_{ii}=\kappa+d^{\mathrm{in}}_i+w\,\mathds 1[i\in R],\quad M_{ij}=-A_{ij}\ (i\neq j),
\end{equation}
with $D_{\mathrm{in}}=\diag(d^{\mathrm{in}}_i)$ the in-degree. $M$ is non-symmetric as soon as some edge is
one-way, so the M-matrix symmetry of Section~\ref{sec:submod} is lost; but $M$ is still strictly
diagonally dominant ($M_{ii}-\sum_{j\neq i}|M_{ij}|=\kappa+w\,\mathds 1[i\in R]\ge\kappa>0$), hence nonsingular with a well-defined coherence
$H=\tr M^{-1}$. Write $\mathrm{recip}(i)=\{j: A_{ij}A_{ji}=1\}$ for the neighbours $i$ both feeds and is fed by.

\begin{proposition}[Directed cavity]\label{prop:directed}
On a locally tree-like directed graph the diagonal resolvent obeys, in the replica-symmetric limit, the
scalar cavity recursion
\begin{equation}\label{eq:dir-cav}
  g_i \;=\; \Big(a_i-\!\!\sum_{j\in\mathrm{recip}(i)}\!\! A_{ij}A_{ji}\,g^{(i)}_j\Big)^{-1},
  \qquad a_i=\kappa+d^{\mathrm{in}}_i+w\,\mathds 1[i\in R],\qquad \frac{H}{N}=\E\,g_i,
\end{equation}
where $g^{(i)}_j$ is the cavity resolvent at $j$ with $i$ removed. The feedback runs over \emph{reciprocated}
edges only: a one-way edge $j\to i$ enters exclusively through the in-degree in $a_i$ and contributes nothing
to the off-diagonal feedback.
\end{proposition}

\begin{proof}
By the Schur complement, $[M^{-1}]_{ii}=\big(M_{ii}-r_i^{\!\top}(M^{(i)})^{-1}c_i\big)^{-1}$. Here $M^{(i)}$
is the minor with row and column $i$ deleted, $r_i$ is the off-diagonal row of $M$ (its entries $M_{ij}$ are
the feeders $j\to i$), and $c_i$ is the off-diagonal column (its entries $M_{ji}$ are the targets $i\to j$). The bilinear term is $\sum_{j,k\neq i}M_{ij}\,[(M^{(i)})^{-1}]_{jk}\,M_{ki}$. On a graph whose
underlying undirected support is locally a tree, $i$ is a cut vertex, so two distinct neighbours $j\neq k$ lie
in different components of $M^{(i)}$ and $[(M^{(i)})^{-1}]_{jk}=0$; only $j=k$ survives, giving
$\sum_{j}M_{ij}M_{ji}\,[(M^{(i)})^{-1}]_{jj}$. The coefficient $M_{ij}M_{ji}=A_{ij}A_{ji}$ is nonzero exactly
when $j\in\mathrm{recip}(i)$, and $[(M^{(i)})^{-1}]_{jj}=g^{(i)}_j$. A one-way edge has $A_{ij}A_{ji}=0$ and drops
out. A $4$-node symbolic check confirms $\partial[M^{-1}]_{ii}/\partial a_k=0$ for every one-way neighbour $k$.
\end{proof}

The reading is sharp: \emph{mutual} trust binds a swarm, one-way authority does not. A corrector that broadcasts
to many agents but listens to none influences their answers, yet its own coherence and the trace it contributes
are set solely by its diagonal: it is a cut vertex with no return path, so no feedback loop forms through it.
Equation~\eqref{eq:dir-cav} specializes the cavity method for sparse non-Hermitian matrices~\cite{NeriMetz} to
the diagonal of the resolvent; what is specific here is that on a tree-like directed graph the coherence
feedback is carried only by reciprocal $2$-cycles.
This also corrects the naive extrapolation of the block form
\eqref{eq:asbm-cav}: summing feedback over \emph{all} incoming neighbours holds only under symmetric support, and even then the reciprocated-only reduction is the tree-level form.

\begin{figure}[t]
  \centering
  \includegraphics[width=0.52\linewidth]{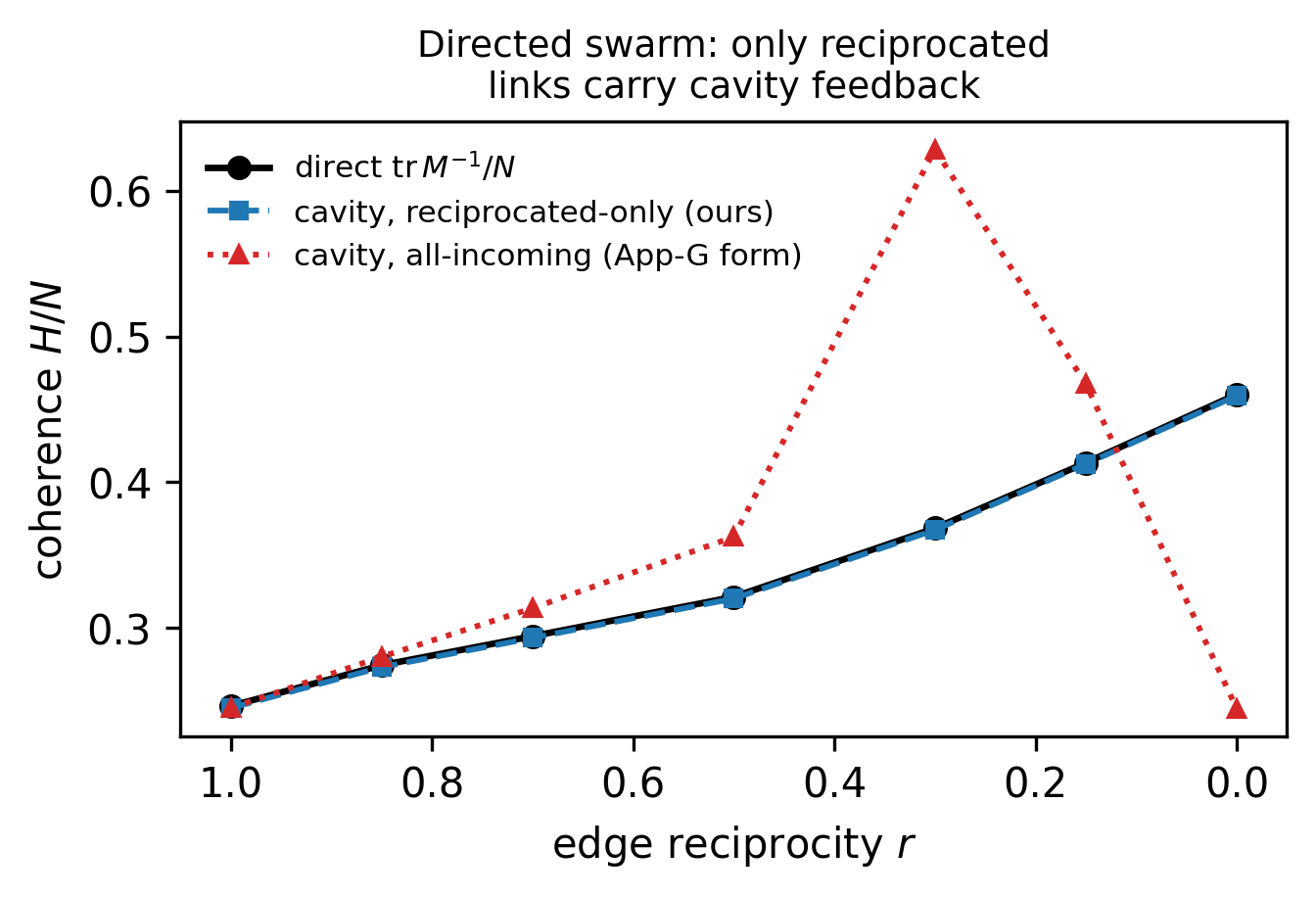}
  \caption{Coherence of a directed swarm as edges turn one-way (reciprocity $r$ swept from fully two-way, right,
  to fully directed, left; sparse graph, $N{=}400$, mean degree $5$, $\kappa{=}0.5$, $w{=}3$, $20\%$ pinned).
  The reciprocated-only cavity \eqref{eq:dir-cav} tracks the exact $\tr M^{-1}/N$ within $0.6\%$ across the whole
  range, including the fully directed graph. Summing feedback over all incoming edges (the block-form
  extrapolation) drifts badly once the graph is directed, by up to ${\sim}70\%$. One-way influence raises the
  coherence cost: removing return paths removes the feedback loops that bind the swarm.}
  \label{fig:directed}
\end{figure}

Solving \eqref{eq:dir-cav} by single-instance message passing on directed Erd\H{o}s--R\'enyi graphs and
comparing against the direct $\tr M^{-1}$ over a reciprocity sweep (Fig.~\ref{fig:directed}) gives agreement
below $0.6\%$ from the two-way graph all the way to the fully directed one, whereas the all-incoming form
drifts by up to ${\sim}70\%$ (worst near $r{\approx}0.3$, and ${\sim}47\%$ at the fully directed graph). The
reciprocated-only cavity is thus the replica-symmetric, tree-level fixed point: it is exact when the underlying
undirected support is locally tree-like, and short directed cycles break the cut-vertex step per instance but
are $O(1/N)$-rare in sparse graphs, so the $N$-averaged trace still matches to within $0.6\%$ even fully
directed. The submodular
placement guarantees of Section~\ref{sec:submod}, which use the symmetry of $M$, do not transfer to the directed
regime and are left open there.

\end{document}